\def\year{2018}\relax
\documentclass[letterpaper]{article} 
\usepackage{aaai18}  
\usepackage{times}  
\usepackage{helvet}  
\usepackage{courier}  
\usepackage{url}  
\usepackage{graphicx}  

\usepackage{soul}
\usepackage{latexsym}
\usepackage{amsmath,amssymb,amsthm}
\usepackage{rotating}
\usepackage{subcaption}
\usepackage{xcolor}

\newtheorem{lemma}{Lemma}

\newtheorem{proposition}{Proposition}
\newtheorem{corollary}{Corollary}
\newtheorem{definition}{Definition}
\newtheorem{example}{Example}

\allowdisplaybreaks

\newif\ifappendix
\appendixtrue 

\begin{document}
%


\title{From Knowledge Graph Embedding to Ontology Embedding?\\ An Analysis of the Compatibility between Vector Space Representations and Rules}


\author{V\'ictor Guti\'errez-Basulto \and Steven Schockaert\\
School of Computer Science and Informatics\\
Cardiff University, Cardiff, UK\\
\{gutierrezbasultov, 
schockaerts1\}@cardiff.ac.uk
}
\maketitle
\begin{abstract}

 Recent years have witnessed the 
 successful application of low-dimensional vector space representations of knowledge graphs to predict missing facts or find erroneous ones.  However, it is not yet well-understood to what extent ontological knowledge, e.g.\ given as a set of (existential) rules, can be embedded in a principled way. 
 To address this shortcoming, in this paper we introduce a general framework based on a view of relations as regions, which allows us to study 
 the compatibility between ontological knowledge and different types of vector space embeddings.
 Our technical contribution is two-fold. First, we show that some of the most popular existing embedding methods are not capable of modelling even very simple types of rules, which in particular also means that they are not able to learn the type of dependencies captured by such rules. Second, we study a model in which relations are modelled as \emph{convex} regions. We show particular that ontologies which are expressed using so-called quasi-chained existential rules can be exactly represented using convex regions, such that any set of facts which is induced using that vector space embedding is logically consistent and deductively closed with respect to the input ontology.
\end{abstract}

\section{Introduction}\label{sec:intro}

Knowledge graphs (KGs), i.e.\ sets of (\textit{subject,\allowbreak predicate,\allowbreak object}) triples, play an increasingly central role in fields such as information retrieval and natural language processing \cite{dong2014knowledge,DBLP:journals/ai/Camacho-Collados16}. A wide variety of KGs are currently available, including carefully curated resources such as WordNet \cite{miller1995wordnet}, crowdsourced resources such as Freebase \cite{bollacker2008freebase}, ConceptNet \cite{speer2017conceptnet} and WikiData \cite{vrandevcic2014wikidata}, and resources that have been extracted from natural language such as NELL \cite{carlson2010toward}. However, despite the large scale of some of these resources, they are, perhaps inevitably, far from complete. This has sparked a large amount of a research on the topic of automated knowledge base completion, e.g.\ random-walk based machine learning models \cite{GardnerM15} and factorization and embedding approaches~\cite{WangMWG17}.
The main premise underlying these approaches is that many plausible triples can be found by exploiting the regularities that exist in a typical knowledge graph. For example, if we know (\textit{Peter Jackson}, \textit{Directed}, \textit{The fellowship of the ring}) and (\textit{The fellowship of the ring}, \textit{Has-sequel}, \textit{The two towers}), we may expect the triple (\textit{Peter Jackson}, \textit{Directed}, \textit{The two towers}) to be somewhat plausible, if we can observe from the rest of the knowledge graph that sequels are often directed by the same person. 

Due to their conceptual simplicity and high scalability, \textit{knowledge graph embeddings} have become one of the most popular strategies for discovering and exploiting such regularities. These embeddings  are $n$-dimensional vector space representations, in which each entity $e$ (i.e.\ each node from the KG) is associated with a vector $\mathbf{e}\in \mathbb{R}^n$ and each relation name $R$ is associated with a scoring function $s_R : \mathbb{R}^n \times \mathbb{R}^n \rightarrow \mathbb{R}$ that encodes information about the likelihood of triples. For the ease of presentation, we will formulate KG embedding models such that $s_{R_1}(\mathbf{e_1},\mathbf{f_1}) < s_{R_2}(\mathbf{e_2},\mathbf{f_2})$ iff the triple $(e_1,R_1,f_1)$ is considered \emph{more} likely than the triple $(e_2,R_2,f_2)$. Both the entity vectors $\mathbf{e}$ and the scoring functions $s_R$ are learned from the information in the given KG. The main assumption is that the resulting vector space representation of the KG is such that it captures the important regularities from the considered domain. In particular, there will be triples $(e,R,f)$ which are not in the original KG, but for which $s_{R}(\mathbf{e},\mathbf{f})$ is nonetheless low. They thus correspond to facts which are plausible, given the regularities that are observed in the KG as a whole, but which are not contained in the original KG. The number of dimensions $n$ of the embedding essentially controls the cautiousness of the knowledge graph completion process: the fewer dimensions, the more regularities can be discovered by the model, but the higher the risk of unwarranted inferences. On the other hand, if the number of dimensions is too high, the embedding may simply capture the given KG, without suggesting any additional plausible triples.


For example, in the seminal TransE model \cite{NIPS20135071}, relations are modelled as vector translations. In particular, the TransE scoring function is given by $s_R(\mathbf{e},\mathbf{f}) = d(\mathbf{e} + \mathbf{r}, \mathbf{f})$, where $d$ is the Euclidean distance and $\mathbf{r} \in \mathbb{R}^n$ is a vector encoding of  the relation name $R$. Another popular model is DistMult \cite{yang2014embedding}, which corresponds to the choice $s_R(\mathbf{e},\mathbf{f}) = -\sum_{i=1}^n e_i r_i f_i$, where we write $e_i$ for the $i^{\textit{th}}$ coordinate of $\mathbf{e}$, and similar for $\mathbf{f}$ and $\mathbf{r}$. 

To date, surprisingly little is understood about the types of regularities that existing embedding methods can capture. In this paper, we are particularly concerned with the types of (hard) rules that such models are capable of representing. To allow us to precisely characterize what regularities are captured by a given embedding, we will consider hard thresholds $\lambda_R$ such that a triple $(e,R,f)$ is considered valid iff $s_R(\mathbf{e},\mathbf{f})\leq \lambda_R$. In fact, KG embeddings are often learned using a max-margin loss function which directly encodes this assumption. The vector space representation of a given relation $R$ can then be viewed as a region $\eta(R)$ in $\mathbb{R}^{2n}$, defined as follows:
$$
\eta(R) = \{\mathbf{e} \oplus \mathbf{f} \,|\, s_R(\mathbf{e},\mathbf{f}) \leq \lambda_R\}
$$
where we write $\oplus$ for vector concatenation. In particular, note that $(e,R,f)$ is considered a valid triple iff $\mathbf{e} \oplus \mathbf{f} \in \eta_R$. Figure \ref{fig:RegionKGview} illustrates the types of regions that are obtained for the TransE and DistMult models.

\begin{figure}
    \centering
    \begin{subfigure}[b]{0.49\linewidth}
   \includegraphics[width=120pt]{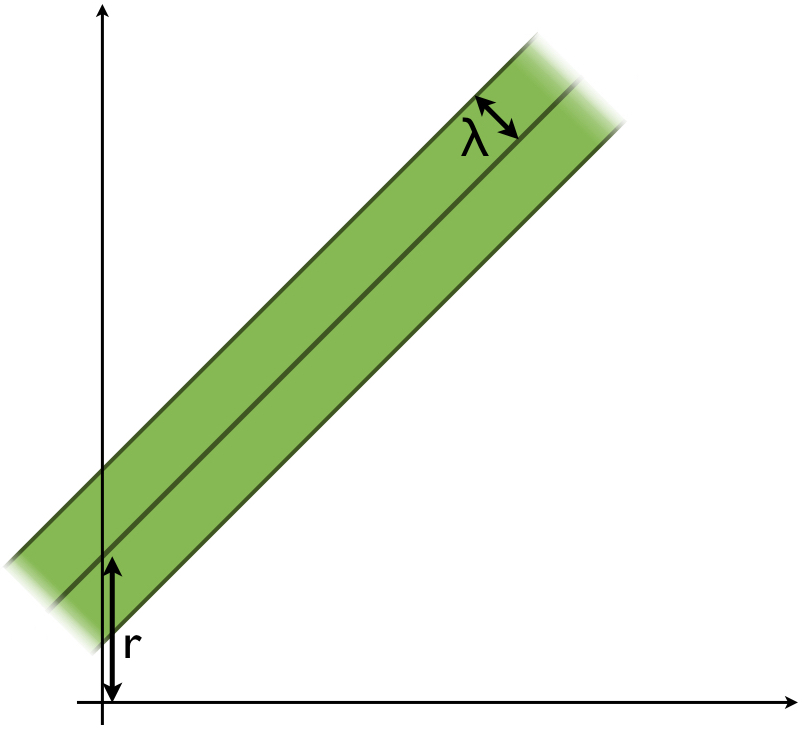}
   \caption{TransE}
   \end{subfigure}
   \begin{subfigure}[b]{0.49\linewidth}
    \includegraphics[width=120pt]{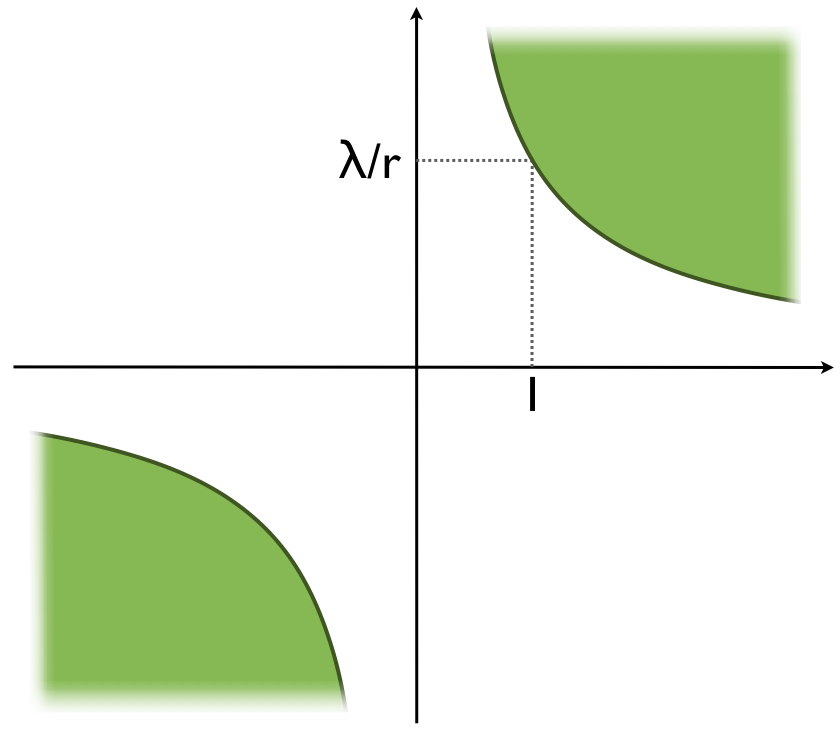}
    \caption{DistMult}
    \end{subfigure}
    \caption{Region based view of knowledge graph embedding models.}
    \label{fig:RegionKGview}
\end{figure}

This region-based view will allow us to study properties of knowledge graph embedding models in a general way, by linking the kind of regularities that a given embedding model can represent to the kind of regions that it considers. Furthermore, the region-based view of knowledge graph embedding also has a number of practical advantages. First, such regions can naturally be defined for relations of any arity, while the standard formulations of knowledge graph embedding models are typically restricted to binary relations. Second, and perhaps more fundamentally, it suggests a natural way to take into account prior knowledge about dependencies between different relations. In particular, for many knowledge graphs, some kind of ontology is available, which can be viewed as a set of rules describing such dependencies. These rules naturally translate to spatial constraints on the regions $\eta_R$. For instance, if we know that $R(X,Y) \rightarrow S(X,Y)$ holds, it would be natural to require that $\eta_R \subseteq \eta_S$. If a knowledge graph embedding captures the rules of a given ontology in this sense, we will call it a \emph{geometric model} of the ontology. By requiring that the embedding of a knowledge graph should be a geometric model of a given ontology, we can effectively exploit the knowledge contained in that ontology to obtain higher-quality representations. 
Indeed, there exists empirical support for the usefulness of  (soft) rules for learning embeddings~\cite{DBLP:conf/emnlp/DemeesterRR16,Niepert16,WangC16,MinerviniDRR17}.
A related advantage of geometric models over standard KG embeddings is that the set of triples which is considered valid based on the embedding is guaranteed to be logically consistent and deductively closed (relative to the given ontology). Finally, since geometric models are essentially ``ontology embeddings'', they could be used for ontology completion, i.e.\ for finding plausible missing rules from the given ontology similar to how standard KG embedding models are used to find plausible missing triples from a KG.

\smallskip\noindent {\bf  Objective and Contributions.} 
The main aim of this paper is to analyze  the implications of choosing a particular type of geometric representation on the  the kinds of logical dependencies that  can be faithfully embedded. To the best of our knowledge, this paper is the first to investigate the expressivity of embedding models in the latter sense.

Our technical contribution is two-fold.
First, we  show that the most popular approaches to KG embedding are actually not compatible with the notion of a geometric model. For instance, as we will see, the representations obtained by DistMult (and its variants) 
can only model a very restricted class of subsumption hierarchies. This is problematic, as it not only means that we cannot impose the rules from a given ontology for learning knowledge graph embeddings, but also that the types of regularities that are captured by such rules cannot be learned from data. 

Second, to overcome the above shortcoming, we propose a novel framework in which relations are modelled as arbitrary convex regions in $\mathbb{R}^k$, with $k$ the arity of the relation. We particularly show that convex geometric models can properly express the class of  so-called \emph{quasi-chained existential rules}. While convex geometric models are thus still not general enough to capture arbitrary existential rules, this particular class does subsume several key ontology languages based on description logics and important fragments of existential rules. Finally, we show that to capture arbitrary existential rules, a further generalization is needed, based on a non-linear transformation of the vector concatenations.




\medskip Missing proofs can be found in a extended version with an appendix  under \url{https://tinyurl.com/yb696el8}


\section{Background}\label{secBackground}

In this section we provide some background on knowledge graph embedding and existential rules.

\subsection{Knowledge Graph Embedding}\label{sec:KGEm}

A wide variety of KG embedding methods have already been proposed, varying mostly in the type of scoring function that is used. One popular class of methods was inspired by the TransE model. In particular, several authors have proposed generalizations of TransE to address the issue that TransE is only suitable for one-to-one relations \cite{TransH,TransR}: if $(e,R,f)$ and $(e,R,g)$ were both in the KG, then the TransE training objective would encourage $f$ and $g$ to be represented as identical vectors. The main idea behind these generalizations is to map the entities to a relation-specific subspace before applying the translation. For instance, the TransR scoring function is given by $s_R(\mathbf{e},\mathbf{f}) = d(M_r \mathbf{e} + \mathbf{r}, M_r \mathbf{f})$, where $M_r$ is an $n\times n$ matrix \cite{TransR}. As a further generalization, in STransE a different matrix is used for the head entity $e$ and for the tail entity $f$, leading to the scoring function $s_R(\mathbf{e},\mathbf{f}) = d(M_r^h \mathbf{e} + \mathbf{r}, M_r^t \mathbf{f})$ \cite{STransE}.

A key limitation of DistMult (cf.\ Section~\ref{sec:intro}) is the fact that it can only model symmetric relations. A natural solution is to represent each entity $e$ using two vectors $\mathbf{e_h}$ and $\mathbf{e_t}$, which are respectively used when $e$ appears in the head (i.e.\ as the first argument) or in the tail (i.e.\ as the second argument). In other words, the scoring function then becomes $s_R = - \sum_i e_i^h r_i f_i^t$, where we write $\mathbf{e_h}=(e_1^h,...,e_n^h)$ and similar for $\mathbf{f_t}$. The problem with this approach is that there is no connection at all between $\mathbf{e_h}$ and $\mathbf{e_t}$, which makes learning suitable representations more difficult. To address this, the ComplEx model \cite{ComplEx} represents entities and relations as vectors of complex numbers, such that $\mathbf{e_t}$ is the component-wise conjugate of $\mathbf{e_h}$. Let us write $\langle \mathbf{a},\mathbf{b},\mathbf{c} \rangle$, with $\mathbf{a},\mathbf{b},\mathbf{c}\in\mathbb{R}^n$, for the bilinear product $\sum_{i=1}^n a_ib_ic_i$. Furthermore, for a complex vector $\mathbf{a} \in \mathbb{C}^n$, we write $\textit{re}(\mathbf{a})$ and $\textit{im}(\mathbf{a})$ for the real and imaginary parts of $\mathbf{a}$ respectively. It can be shown \cite{kazemi2018simple} that the scoring function of ComplEx is equivalent to 
\begin{align*}
s_R(e,f) =& - \langle \textit{re}(\mathbf{e}),\textit{re}(\mathbf{r}),\textit{re}(\mathbf{f})\rangle - \langle \textit{re}(\mathbf{e}),\textit{im}(\mathbf{r}),\textit{im}(\mathbf{f})\rangle  \\ &-\langle\textit{im}(\mathbf{e}),\textit{re}(\mathbf{r}),\textit{im}(\mathbf{f})\rangle + \langle \textit{im}(\mathbf{e}),\textit{im}(\mathbf{r}),\textit{re}(\mathbf{f})\rangle\notag
\end{align*}
 Recently, in \cite{kazemi2018simple}, a simpler approach was proposed to address the symmetry issue of DistMult. The proposed model, called SimplE, avoids the use of complex vectors. In this model, the  DistMult scoring function is used with a separate representation for head and tail mentions of an entity, but for each triple $(e,R,f)$ in the knowledge graph,  the triple $(f,R^{-1},e)$ is additionally considered. This means that each such triple affects the representation of $\mathbf{e_h}$, $\mathbf{e_t}$, $\mathbf{f_h}$ and $\mathbf{f_t}$, and in this way, the main drawback of using separate representations for head and tail mentions is avoided. 

The RESCAL model \cite{nickel2011three} uses a bilinear scoring function $s_R(e,f)= - \mathbf{e}^T M_r\mathbf{f}$, where the relation $R$ is modelled as an $n \times n$ matrix $M_r$. Note that DistMult can be seen as a special case of RESCAL in which only diagonal matrices are considered. Similarly, it is easy to verify that ComplEx also corresponds to a bilinear model, with a slightly different restriction on the type of considered matrices. Without any restriction on the type of considered matrices, however, the RESCAL model is prone to overfitting. The neural tensor model (NTN), proposed in \cite{NTN} further generalizes RESCAL by using a two-layer neural network formulation, but similarly tends to suffer from overfitting in practice.

\smallskip\noindent \textbf{Expressivity.}
Intuitively, the reason why KG embedding models are able to identify plausible triples is because they can only represent knowledge graphs that exhibit a certain type of regularity. They can be seen as a particular class of dimensionality reduction methods: the lower the number of dimensions $n$, the stronger the KG model enforces some notion of regularity (where the exact kind of regularity depends on the chosen KG embedding model). However, when the number of dimensions is sufficiently high, it is desirable that any KG can be represented in an exact way, in the following sense: for any given set of triples $P=\{(e_1,R_1,f_1),..., (e_m,R_m,f_m)\}$ which are known to be valid and any set of triples $N=\{(e_{m+1},R_{m+1},f_{m+1}),...(e_k,R_k,f_k)\}$ which are known to be false, given a sufficiently high number of dimensions $n$, there always exists an embedding and thresholds $\lambda_R$ such that
\begin{align}
&\forall (e,R,f)\in P \,.\, s_{R}(\mathbf{e},\mathbf{f})\leq\lambda_R\label{eqFullyExpressiveA}\\
&\forall (e,R,f)\in N \,.\, s_{R}(\mathbf{e},\mathbf{f})> \lambda_R\label{eqFullyExpressiveB}
\end{align}
A KG embedding model is called \emph{fully expressive} \cite{kazemi2018simple} 
if \eqref{eqFullyExpressiveA}--\eqref{eqFullyExpressiveB} can be guaranteed for any disjoint sets of triples $P$ and $N$. If a KG embedding model is not fully expressive, it means that there are \textit{a priori} constraints on the kind of knowledge graphs that can be represented, which can lead to unwarranted inferences when using this model for KG completion. In contrast, for fully expressive models, the types of KGs that can be represented is determined by the number of dimensions, which is typically seen as a hyperparameter, i.e.\ this number is tuned separately for each KG to avoid (too many) unwarranted inferences.

It turns out that translation based methods such as TransE, STransE and related generalizations are not fully expressive \cite{kazemi2018simple}, and in fact put rather severe restrictions on the types of relations that can be represented in the sense of \eqref{eqFullyExpressiveA}--\eqref{eqFullyExpressiveB}. For instance, it was shown in \cite{kazemi2018simple} that translation based methods can only fully represent a knowledge graph $G$  if each of its relations $R$ satisfies the following properties for every subset of entities $S$:
\begin{enumerate}
    \item If $R$ is reflexive over $S$, then $R$ is also symmetric and transitive over $S$.
    \item If $\forall s\in S\,.\,(e,R,s)\in G$  and  $\exists s\in S\,.\,(f,R,s)\in G$ then we also have  $\forall s\in S\,.\,(f,R,s)\in G$.
\end{enumerate} 
 However, both ComplEx and SimplE have been shown to be fully expressive.

\smallskip\noindent \textbf{Modelling Textual Descriptions.}
Several methods have been proposed which aim to learn better knowledge graph embeddings by exploiting textual descriptions of entities \cite{DBLP:conf/emnlp/ZhongZWWC15,xie2016representation,SSP} or by extracting information about the relationship between two entities from sentences mentioning both of them \cite{toutanova2015representing}. Apart from improving the overall quality of the embeddings, a key advantage of such approaches is that they allow us to predict plausible triples involving entities which do not occur in the initial knowledge graph. 

\subsection{Existential Rules} 

Existential rules (a.k.a.\ Datalog$^\pm$) are a family of rule-based formalisms for  modelling ontologies.  An existential rule is a datalog-like rule with existentially quantified variables in the head, i.e.\ it extends traditional datalog with \emph{value invention}.  
As a consequence, existential rules describe not only constraints on the currently available knowledge or data, but also \emph{intentional} knowledge about the domain of discourse. 
The appeal of existential rules  comes from the fact that they are extensions of 
the prominent  $\mathcal{EL}$ and \textsl{DL-Lite} families of description logics (DLs)~\cite{DLBook}. For instance,  existential rules  can describe $k$-ary relations, while DLs are constrained to unary and binary relations.  

\noindent {\bf Syntax.} Let ${\bf C}, {\bf N}$ and ${\bf V}$ be   infinite disjoint sets of \emph{constants, (labelled) nulls} and \emph{variables}, respectively. A \emph{term} $t$ is an element in  ${\bf C} \cup {\bf N}  \cup{\bf V}$; an \emph{atom} $\alpha$ is an expression of the form $R(t_1, \ldots, t_n)$, where $R$ is  a \emph{relation name (or predicate)} with \emph{arity} $n$ and terms $t_i$. We denote with $\mathsf{terms}(\alpha)$ the set $\{t_1, \ldots, t_n\}$ and with  $\mathsf{vars}(\alpha)$ the set $\mathsf{terms}(\alpha) \cap {\bf V}$. 
An \emph{existential rule} $\sigma$ is an expression of the form
\begin{equation}\label{eqerule}
 B_1 \land \ldots \land B_n \rightarrow \exists X_1, \ldots ,X_j.H_1 \land \ldots \land H_k, 
 \end{equation}
 where $B_1, \ldots B_n$ for $n \geq 0$, 
 $H_1, \ldots, H_k$ for $k \geq 1$, 
 are atoms with 
 terms in ${\bf C} \cup {\bf V}$ and $X_m \in \mathbf V$ for  $1 \leq m\leq j$. From here on, we assume w.l.o.g that $k=1$~\cite{CaliGK13}; we omit in this case the subindex. 
 We use $\mathsf{body}(\sigma)$ and $\mathsf{head}(\sigma)$ to refer to $\{B_1, \ldots, B_n\}$ and $\{H\}$, respectively. We call $\mathsf{evars}(\sigma)=\{X_1, \ldots, X_j\}$ the \emph{existential variables of $\sigma$}; if $\mathsf{evars}(\sigma) = \emptyset$, $\sigma$ is called a \emph{datalog rule}. We further allow \emph{negative constraints} (or simply \emph{constraints}) which are expressions of the form    
 %
 $B_1 \land \ldots \land B_n \rightarrow \bot,$ 
 %
 where the $B_i$s are as above and $\bot$ denotes the truth constant \emph{false}. A finite set $\Sigma$ of existential rules and constraints is called an \emph{ontology}; and a \emph{datalog program} if $\Sigma$ contains only datalog rules and constraints. 

\smallskip 
  Let $\mathfrak R$ be a set of relation names. A \emph{database $D$} is a finite set of \emph{facts} over $\mathfrak R$, i.e.\ atoms with terms in $\bf C$.  A \emph{knowledge base (KB)} $\mathcal K$ is a pair $(\Sigma, D)$ with $\Sigma$ an ontology (or a datalog program) and $D$ a database. 
 
  \smallskip\noindent \textbf{Semantics.}
  An \emph{interpretation $\mathcal I$ over $\mathfrak{R}$} is a (possibly infinite) set of atoms over $\mathfrak R$ with terms in ${\bf C} \cup {\bf N}$.  An interpretation $\mathcal I$ is a \emph{model  of $\Sigma$} if it satisfies all rules and constraints: $\{B_1, \ldots, B_n \}\subseteq \mathcal I$ implies $\{H\} \subseteq \mathcal I$ for every $\sigma$ defined as above in $\Sigma$, where existential variables can be witnessed by constants or labelled nulls, and $\{B_1, \ldots, B_n\} \not \subseteq \mathcal I$ for all constraints  defined as above in $\Sigma$;  it is a \emph{model of a database} $D$ if $D \subseteq \mathcal I$; it is a model of a KB $\mathcal K = (\Sigma, D)$, written $\mathcal I \models \mathcal K$, if it is a model of $\Sigma$ and $D$. We say that a KB $\mathcal K$ is satisfiable if it has a model. We  refer to elements in ${\bf C} \cup {\bf N}$ simply as \emph{objects}, call atoms $\alpha$  containing only objects as terms \emph{ground}, and denote with
  $\mathfrak O(\mathcal I)$ the set of all objects occurring in $\mathcal I$.
  
  \begin{example}\label{Ex:1}
  Let  $D =\{\textit{Wife}(\textit{anna}), \textit{Wife}(\textit{marie})\}$ be a database and $\Sigma$ an ontology composed by the rules:
  %
  %
 \begin{align}
\textit{Wife}(X) \wedge \textit{Married}(X,Y) \rightarrow \textit{Husband}(Y)\label{rul:0}\\
%
\textit{Wife}(Y) \rightarrow \exists X\,.\, \textit{Husband}(X) \wedge \textit{Married}(X,Y) \label{rul:1}\\
%
%
 \textit{Husband}(X) \land \textit{Wife}(X) \rightarrow \bot\label{rul:2}
  \end{align}

\noindent Then, an example of a model  of  $\mathcal K = (\Sigma, D)$ is the set of atoms  $D \cup \{\textit{Husband}(o_1),\textit{Husband}(o_2),  \textit{Married}(o_1, \textit{anna}),$ $ \textit{Married}(o_2, \textit{marie}) \}$ where $o_i$ are labelled nulls. Note that e.g.\
$\{\textit{Married}(\textit{anna},\textit{marie}), \textit{Husband}(\textit{marie}) \}$ is not included in any model of $\mathcal K$ due to~\eqref{rul:2}.
  \end{example}
  

\smallskip\noindent \textbf{Notation.}  We use $a, b,c, a_1, \ldots$ for constants and $X,Y,Z, X_1, \ldots$ for variables. 
We write $\mathfrak{R}_k$ for the set of relation names from $\mathfrak{R}$ which have arity $k$.
Given a  KB $\mathcal K$, we use $\bf C(\mathcal K)$, $\mathfrak R(\mathcal K)$ and $\mathfrak{R}_k(\mathcal K)$ to denote, respectively,  the set of constants, relation names and $k$-ary relation names occurring in  $\mathcal K$.   For vectors $\mathbf{x}=(x_1,...,x_m)$ and $\mathbf{y}=(y_1,...,y_k)$, we denote their concatenation by $\mathbf{x} \oplus \mathbf{y} = (x_1,...,x_m,y_1,...,y_k)$. 


\section{Geometric Models}\label{secProblemFormulation}


In this section, we formalize how regions can be used for representing relations, and what it means for such representations to satisfy a given knowledge base. The resulting formalization will allow us to study the expressivity of knowledge graph embedding models. It will also provide the foundations of a framework for  knowledge base completion, based on embeddings that are jointly learned from a given database and ontology. We first define the geometric counterpart of an interpretation.
 
\begin{definition}[Geometric interpretation]\label{defGeometricInterpretation}
Let $\mathfrak R$ be a set of relation names and ${\bf X}\subseteq {\bf C} \cup {\bf N}$ be a  set of objects. An \emph{$m$-dimensional geometric interpretation $\eta$ of $(\mathfrak{R},{\bf X})$} assigns to each $k$-ary relation name $R$ from $\mathfrak{R}$ a region 
$\eta(R) \subseteq \mathbb{R}^{k\cdot m}$ and to each object $o$ from ${\bf X}$ a vector $\eta( o) \in \mathbb{R}^m$.
\end{definition}


\begin{figure}
    \centering
    \includegraphics[width=0.6\columnwidth]{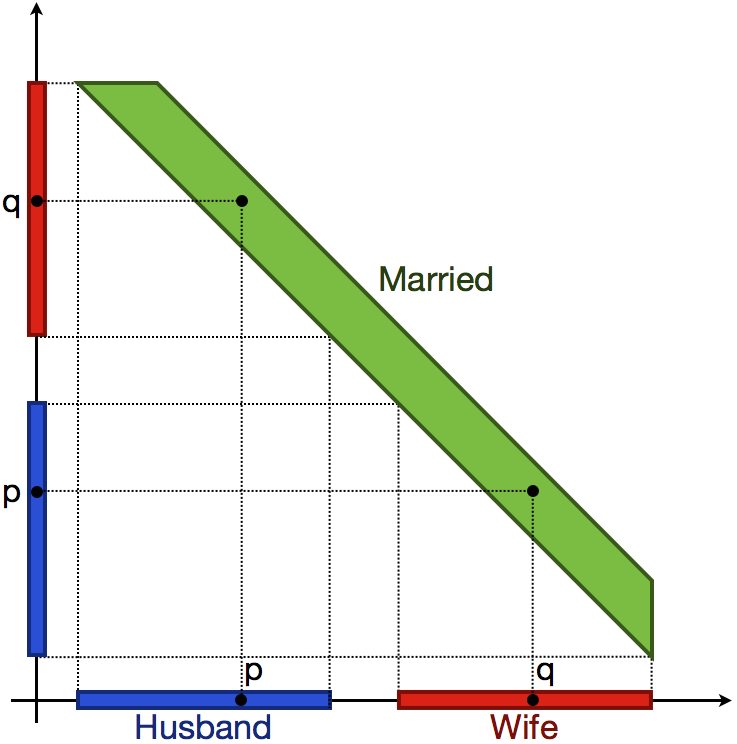}
    \caption{A geometric model of the KB from Example \ref{Ex:1}.}
    \label{fig:Married}
\end{figure}
\noindent An example of a 1-dimensional geometric interpretation of $(\{\textit{Husband},\textit{Wife},\textit{Married}\},\{p,q\})$ is depicted in Figure \ref{fig:Married}. Note that in this case, the unary predicates \textit{Husband} and \textit{Wife} are represented as intervals, whereas the binary predicate \textit{Married} is represented as a convex polygon in $\mathbb{R}^2$. We now define what it means for a geometric interpretation to satisfy a ground atom.

\begin{definition}[Satisfaction of ground atoms]\label{def:sat}
Let $\eta$ be an $m$-dimensional geometric interpretation of $(\mathfrak{R},{\bf X})$, $R \in \mathfrak{R}_k$ and  $o_1,...,o_k\in {\bf X}$. We say that $\eta$ \emph{satisfies} a ground atom $R(o_1,..., o_k)$, written $\eta \models R(o_1,..., o_k)$, if $\eta(o_1) \oplus ... \oplus \eta(o_k) \in \eta(R)$. 
\end{definition}

\noindent For ${\bf Y}\subseteq {\bf X}$, we will write $\phi({\bf Y},\eta)$ for the set of ground atoms over {\bf Y} which are satisfied by $\eta$, i.e.:
\begin{align*}
 \{R(o_1,...,o_k) \,|\, &R\in \mathfrak{R}_k, o_1,...,o_k \in {\bf Y},\eta \models R( o_1,...,o_k)\}
\end{align*}
If ${\bf Y} = {\bf X}$, we also abbreviate $\phi({\bf Y},\eta)$ as $\phi(\eta)$. For example, if $\eta$ is the geometric interpretation from Figure \ref{fig:Married}, we find:
\begin{align*}
\phi(\eta) {=} \{\textit{Husband}(p),\textit{Wife}(p),\textit{Married}(p,q),\textit{Married}(q,p)\}     
\end{align*}

\noindent The notion of satisfaction in Definition~\ref{def:sat} can be extended to propositional combinations of ground atoms in the usual way. Specifically, $\eta$ satisfies a rule $B_1 \wedge ... \wedge B_n \rightarrow C$, with $B_1,...B_n,C$ ground atoms, if $\eta\models C$ or $\{B_1,...,B_n\} \not\subseteq \phi(\eta)$. Now consider the case of a non-ground rule, e.g.:
\begin{align}\label{eqExampleHornRule}
R(X,Y) \wedge S(Y,Z) \rightarrow T(X,Z)
\end{align}
Intuitively what we want to encode is whether $\eta$ satisfies every possible grounding of this rule, i.e.\ whether for any objects $o_x,o_y,o_z$ such that $\eta(o_x)\oplus \eta(o_y) \in \eta(R)$ and $\eta(o_x)\oplus \eta(o_z) \in \eta(S)$ it holds that $\eta(o_x)\oplus \eta(o_z) \in \eta(T)$. However, since an important aim of vector space representations is to enable inductive generalizations, this property of $\eta$ should not only hold for the constants occurring in the given knowledge base, but also for any possible constants whose representation we might learn from external sources \cite{DBLP:conf/emnlp/ZhongZWWC15,xie2016representation,SSP}. As a result, we need to impose the following stronger requirement for $\eta$ to satisfy \eqref{eqExampleHornRule}: for every $\mathbf{x},\mathbf{y},\mathbf{z} \in \mathbb{R}^m$ such that $\mathbf{x} \oplus \mathbf{y} \in \eta(R)$ and $\mathbf{y} \oplus \mathbf{z} \in \eta(S)$, it has to hold that $\mathbf{x} \oplus \mathbf{z} \in \eta(S)$. Note that a rule like \eqref{eqExampleHornRule} thus naturally translates into a spatial constraint on the representation of the relation names. Finally, let us consider an existential rule:
\begin{align}\label{eqExampleHornRule1}
R(X,Y) \rightarrow \exists Z\,.\, S(X,Y,Z)
\end{align}
For $\eta$ to be a model of this rule, we require that for every $\mathbf{x},\mathbf{y}\in \mathbb{R}^m$ such that $\mathbf{x}\oplus \mathbf{y} \in \eta(R)$ there has to exist a $\mathbf{z}\in \mathbb{R}^m$ such that $\mathbf{x}\oplus \mathbf{y}\oplus \mathbf{z} \in \eta(S)$. These intuitions are formalized in the following definition of a geometric model.

\begin{definition}\label{defGeometricModel} 
Let $\mathcal{K} = (\Sigma, D)$ be a knowledge base and $\mathfrak O$ a (possibly infinite) set of objects. 
A geometric interpretation $\eta$ of $(\mathfrak{R}(\mathcal K),\mathfrak O)$ is called an \emph{$m$-dimensional geometric model of $\mathcal{K}$} if 
\begin{enumerate}
    \item  $\phi(\eta)= \mathcal M$, for some model  $\mathcal M$ of $\mathcal K$, and 
    \item
    for any set of points $\{\mathbf{v_1},...,\mathbf{v_n}\} \subseteq \mathbb{R}^m$, $\eta$ can be extended to a geometric interpretation  $\eta^*$ 
    such that
    \begin{enumerate}
    
        \item for each $i\in \{1,...,n\}$ there is a fresh constant $c_i \in \mathbf{C} \setminus \mathfrak O (\mathcal M)$ such that $\eta^*(c_i)=\mathbf{v_i}$, 
        \item $\phi(\eta^*)= \mathcal M'$ for some model  $\mathcal M'$ of  $(\Sigma, D \cup \phi(\eta^*))$.
    \end{enumerate}
\end{enumerate}

 
\end{definition}
\noindent
The first point in Definition \ref{defGeometricModel} ensures that we can view geometric models as geometric representations of classical models.
The second point in Definition \ref{defGeometricModel} ensures that we can use geometric models to introduce objects from external sources, without introducing any inconsistencies. It captures the fact that the logical dependencies between the relation names encoded in $\Sigma$ should be properly captured by the spatial relationships between their geometric representations, as was illustrated in \eqref{eqExampleHornRule1}.  Naturally, $\mathcal M'$ might contain additional (in comparison to $\mathcal M$) nulls to witness existential demands over the new constants. For datalog programs, however, $\eta^*$ is completely determined by $\eta$ and the fact that $\eta^*(c_i)=\mathbf{v_i}$ for $1\leq i\leq n$, that is, only Conditions 1 is necessary.
For instance, the geometric interpretation depicted in Figure \ref{fig:Married} is a geometric model of the rules from Example \ref{Ex:1}.


\smallskip \noindent 
{\bf Practical Significance of Geometric Models.} The framework presented in this section offers several key advantages over standard KG embedding methods. First, it allows us to take into account a given ontology when learning the vector space representations, which should lead to higher-quality representations, and thus more faithful predictions, in cases where such an ontology is available. Also note that the region based framework can be applied to relations of any arity. Conversely, the framework also naturally allows us to obtain plausible rules from a learned geometric model, as this geometric model may (approximately) satisfy rules which are not entailed by the given ontology. Moreover, our framework allows for a tight integration of deductive and inductive modes of inference, as the  facts and rules that are satisfied by a geometric model are deductively closed and logically consistent. 

\smallskip
\noindent{\bf Modelling Relations as Convex Regions.} 
While, in principle, arbitrary subsets of $\mathbb{R}^{k\cdot m}$ can be used for representing $k$-ary relations, in practice the type of considered regions will need to be restricted in some way. This is needed to ensure that the regions can be efficiently learned from data and can be represented compactly. Moreover, the purpose of using vector space representations is to enable inductive inferences, but this is only possible if we impose sufficiently strong regularity conditions on the representations. For this reason, in this paper we will particularly focus on \emph{convex} geometric interpretations, i.e.\ geometric interpretations in which each relation is represented using a convex region. While this may seem like a strong assumption, the vast majority of existing KG embedding models in fact learn representations that correspond to such convex geometric interpretations. Moreover, when learning regions in high-dimensional spaces, strong assumptions such as convexity are needed to avoid overfitting, especially if the amount of training data is limited. Finally, the use of convex regions is also in accordance with cognitive models such as conceptual spaces \cite{Gardenfors:conceptualSpaces}, and more broadly with experimental findings in psychology, especially in cases where we are presented with few training examples \cite{ROSSEEL2002178}.

\medskip 
One may wonder whether it is possible to go further and restrict attention e.g.\ to convex models that are induced by vector translations. For instance, we could consider regions which are such that $\mathbf{x}\oplus \mathbf{y} \in \eta(R)$ means that we also have $\mathbf{u}\oplus \mathbf{v} \in \eta(R)$ whenever $\mathbf{y}-\mathbf{x} = \mathbf{v}-\mathbf{u}$, i.e.\ only the vector difference between $\mathbf{x}$ and $\mathbf{y}$ matters. Note that TransE and most of its generalizations aim to learn representations that correspond to such regions. Alas,  as the next example illustrates, such translation-based regions do not have the desired generality, in the sense that they cannot properly capture even simple rules.
\begin{example}\label{exampleIntro}
For instance, consider rules \eqref{rul:1}-\eqref{rul:2} in Example~\ref{Ex:1}.
For the ease of presentation, let us write $C_H$ for $\eta(\textit{Husband})$ and $C_W$ for $\eta(\textit{Wife})$, i.e.\ we assume that $\textit{Husband}(a)$ holds for a constant $a$ iff $\mathbf{a} \in C_H$. Let us furthermore assume that $C_H$ and $C_W$ are convex. We will also assume that a translation-based region is used to represent \textit{Married}. Note that in such a case, the region $\eta(\textit{Married})$ in $\mathbb{R}^{2n}$ can be characterized by a region $C_M$ in $\mathbb{R}^n$ such that $\textit{Married}(a,b)$ holds iff $\mathbf{b}-\mathbf{a} \in C_M$. To capture the logical dependencies encoded by the  rules, the following spatial relationships would then have to hold:
\begin{align}
C_H &\supseteq \{\mathbf{p} + \mathbf{r} \,|\, \mathbf{p} \in C_W, \mathbf{r} \in C_M\} \label{eqSpatialConstraint1}\\
C_W &\subseteq \{\mathbf{p} + \mathbf{r} \,|\, \mathbf{p} \in C_H, \mathbf{r} \in C_M\} \label{eqSpatialConstraint2}
\end{align}
However, \eqref{eqSpatialConstraint1} and \eqref{eqSpatialConstraint2} entail\footnote{Indeed, suppose that $\textbf{q} \in C_W$, then by \eqref{eqSpatialConstraint2} there must exist some $\mathbf{p} \in C_H$ and $\mathbf{r} \in C_M$ such that $\mathbf{q} = \mathbf{p} + \mathbf{r}$. By \eqref{eqSpatialConstraint1} we furthermore have $\mathbf{q} + \mathbf{r} \in C_H$. Since $\mathbf{q}$ is between $\mathbf{p}$ and $\mathbf{q} + \mathbf{r}$, both of which belong to $C_H$, by the convexity of $C_H$ it follows that $q\in C_H$.} that $C_W \subseteq C_H$. 
Since, by rule~\eqref{rul:0}, the concepts \textit{Wife} and \textit{Husband} are disjoint, we would have to choose $C_W=C_H=\emptyset$ and would not be able to represent any instances of these concepts.
\end{example}

\noindent It is perhaps not surprising that translation based representations are not suitable for modelling rules, since they are already known not to be fully expressive in the sense of \eqref{eqFullyExpressiveA}--\eqref{eqFullyExpressiveB}. As we discussed in Section \ref{sec:KGEm}, there are several bilinear models which are known to be fully expressive, and which may thus be thought of as more promising candidates for defining suitable types of regions. We  address whether bilinear models are able to represent ontologies in the next section.


\section{Limitations of Bilinear Models}

As already mentioned, translation based approaches incur rather severe limitations on the kinds of databases and ontologies that can be modelled. In this section, we show that while bilinear models are fully expressive, and can thus model any database, they are not suitable for modelling ontologies. This motivates the need for novel embedding methods, which are better suited at modelling ontologies; this will be the focus of the next section.

\smallskip Let us consider the following common type of rules:
\begin{align}
R(X,Y) \rightarrow S(X,Y) \label{eqArg1Rule}
\end{align}
and  a bilinear model in which each relation name $R$ is associated with an $n\times n$ matrix $M_r$ and a threshold $\lambda_r$. We then say that \eqref{eqArg1Rule} is satisfied if for each $\mathbf{e},\mathbf{f}\in \mathbb{R}^n$, it holds that:
\begin{align}\label{eqArg2RuleBilinear}
(\mathbf{e}^T  M_r  \mathbf{f} \geq \lambda_r) \Rightarrow (\mathbf{e}^T M_s \mathbf{f} \geq \lambda_s)
\end{align}
where $\mathbf{e}^T$ denotes the transpose of $\mathbf e$.
It turns out that bilinear models are severely limited in how they can model sets of rules of the form \eqref{eqArg1Rule}. This limitation stems from the following result.

\begin{proposition}\label{propLimitationBilinear}
Suppose that \eqref{eqArg2RuleBilinear} is satisfied for the matrices $M_r,M_s$ and some thresholds $\lambda_r,\lambda_s$. Then there exists some $\alpha\geq 0$ such that $M_r = \alpha M_s$.
\end{proposition}

\noindent If $\alpha=0$ then the rule \eqref{eqArg1Rule} must be satisfied trivially, in the sense that the following rule is also satisfied for the matrix $M_s$ and threshold $\lambda_s$:
$$
\top \rightarrow S(X,Y)
$$
Let us consider the case where $\alpha>0$. Note that  for the thresholds $\lambda_r$ and $\lambda_s$ we only need to consider the values -1 and 1 since other thresholds can always be simulated by rescaling the matrices $M_r$ and $M_s$. Now assume that the following rules are given:
\begin{align*}
    R_1(X,Y) &\rightarrow S(X,Y)\\
     &...\\
    R_k(X,Y) &\rightarrow S(X,Y) 
\end{align*}
By Proposition \ref{propLimitationBilinear}, we know that for $i\in\{1,...,k\}$ there is some $\alpha_i$ such that  $M_{r_i} = \alpha_i M_{s_i}$. If $\lambda_{r_i}=\lambda_{r_j}$ we thus have that either the rule $R_i(X,Y) \rightarrow R_j(X,Y)$ or the rule $R_j(X,Y) \rightarrow R_i(X,Y)$ is satisfied (depending on whether $\alpha_i \geq 1$ and on whether $\lambda_{r_i}$ is 1 or -1).  
This means in particular that we can always find two rankings $R_{\tau_1},...,R_{\tau_p}$ and $R_{\sigma_1},...,R_{\sigma_q}$ such that $\{R_1,...,R_k\} = \{R_{\tau_1},...,R_{\tau_p},R_{\sigma_1},...,R_{\sigma_q}\}$ and:
\begin{align*}
\forall 1\leq i < p \,.\,   R_{\tau_i}(X,Y) &\rightarrow R_{\tau_{i+1}}(X,Y)\\
\forall 1\leq i < q \,.\,   R_{\sigma_i}(X,Y) &\rightarrow R_{\sigma_{i+1}}(X,Y)
\end{align*}
This clearly puts drastic restrictions on the type of subsumption hierarchies that can be modelled using bilinear models. Moreover, these limitations carry over to DistMult and ComplEx, as these are particular types of bilinear models. Due to the close links between DistMult and SimplE, it is also easy to see that the latter model has the same limitations.

In fact, the use of different vectors for head and tail mentions of entities in the SimplE model leads to even further limitations. To illustrate this, let us consider a rule of the following form:
\begin{align}
R(X,Y) \wedge S(Y,Z) \rightarrow T(X,Z) \label{eqArg2Rule}
\end{align}
where we say that the SimplE representation defined by the vectors $\mathbf{r},\mathbf{ri},\mathbf{s},\mathbf{si},\mathbf{t},\mathbf{ti}$ and corresponding thresholds $\lambda_r,\lambda_{ri},\lambda_s,\allowbreak\lambda_{si},\lambda_t,\lambda_{ti}$ satisfies \eqref{eqArg2Rule} if for all entity vectors $\mathbf{e_h},\mathbf{e_t},\mathbf{f_h},\mathbf{f_t},\mathbf{g_h},\mathbf{g_t}$ it holds that:
\begin{align}
&\langle\mathbf{e_h}, \mathbf{r}, \mathbf{f_t} \rangle \geq \lambda_r
\wedge \langle\mathbf{f_h}, \mathbf{ri}, \mathbf{e_t} \rangle \geq \lambda_{ri}\label{eqTripleImplicationCompositionRule}\\
&\wedge \langle\mathbf{f_h}, \mathbf{s}, \mathbf{g_t} \rangle \geq \lambda_s
\wedge \langle\mathbf{g_h}, \mathbf{si}, \mathbf{f_t} \rangle \geq \lambda_{si}\notag\\
&\quad\quad\quad\quad\Rightarrow \langle\mathbf{e_h}, \mathbf{t}, \mathbf{g_t} \rangle \geq \lambda_t
\wedge \langle\mathbf{g_h}, \mathbf{ti}, \mathbf{e_t} \rangle \geq \lambda_{ti}\notag
\end{align}
Then we can show the following result.
\begin{proposition}\label{propLimitationSimplE}
Suppose $\mathbf{r},\mathbf{ri},\mathbf{s},\allowbreak \mathbf{si},\allowbreak \mathbf{t},\mathbf{ti}$ and $\lambda_r,\lambda_{ri},\lambda_s,\allowbreak\lambda_{si},\lambda_t,\lambda_{ti}$ define a SimplE representation satisfying \eqref{eqArg2Rule}. Then one of the following two rules is satisfied as well:
\begin{align}
R(X,Y) \wedge S(Y,Z) \rightarrow \bot \label{eqPropSimlpECompositionRule1}\\
\top \rightarrow  T(X,Z) \label{eqPropSimlpECompositionRule2}
\end{align}
\end{proposition}

 \section{Relations as Arbitrary Convex Regions}\label{secMainResult}
In this section we  consider arbitrary convex geometric models, and show that they can correctly represent a large class of existential rules.
We particularly show that  KBs $\mathcal K$ based on \emph{quasi-chained rules} are properly captured by convex geometric models, in the sense that for each finite  model $\mathcal{I}$ of $\mathcal K$, there exists a convex geometric model $\eta$ such that $\mathcal{I} = \phi(\eta)$. 

 \medskip \noindent \textbf{Quasi-chained Rules.} 
 We say that  an existential rule $\sigma$, defined as in~\eqref{eqerule} above,  is  \emph{quasi-chained (QC)} if for all $1 \leq i \leq n$
$$|(\mathsf{vars}(B_1) \cup ... \cup \mathsf{vars}(B_{i-1})) \cap\mathsf{vars}(B_{i})| \leq 1$$
An ontology  is quasi-chained  if all its rules are either quasi-chained or  quasi-chained negative constraints. 

\smallskip Note that quasi-chainedness is a natural and useful restriction.  Quasi-chained rules are indeed closely related to the well-known chain-datalog fragment of datalog~\cite{Shmueli87,UllmanG88} in which important properties, e.g. reachability, are still expressible.  Furthermore, prominent Horn description logics can be expressed using decidable fragments of   quasi-chained existential  rules. For example,  $\mathcal{ELHI}$ ontologies\footnote{We assume they are in a suitable normal form~\cite{DLBook}} can be embedded into the  \emph{guarded fragment}~\cite{CaliGK13} 
of QC existential rules. Further,  QC existential rules subsume \emph{linear existential rules}, which only allow rule bodies that consist of a single atom and capture a $k$-ary extension of \textsl{DL-Lite}$_\mathcal R$.

\medskip We next show the announced result that geometric models properly capture quasi-chained ontologies. 


\begin{proposition}\label{propmain}
Let $\mathcal K=(\Sigma, D)$, with $\Sigma$ a quasi-chained ontology, and let $\mathcal{M}$ be a finite model of $\mathcal K$. Then $\mathcal{K}$ has a convex geometric model $\eta$ such that $\phi(\eta)=\mathcal{M}$.
\end{proposition}

\noindent To clarify the intuitions behind this proposition, we show how an $m$-dimensional geometric model $\eta$ satisfying $\phi(\eta)=\mathcal{M}$ can be constructed, where $m= |\mathfrak O(\mathcal M)|$. Let $x_1, \ldots, x_m$ be an enumeration of the elements in $\mathfrak O(\mathcal M)$, 
 then  for each $x_i$,  $\eta(x_i)$ is defined as the vector in $\mathbb{R}^m$ with  value $1$ in the $i^{th}$ coordinate  and $0$ in all others. Further, for each $R \in \mathfrak R_k(\mathcal K)$, we define $\eta(R) $ as follows, where $\text{CH}$ denotes the convex-hull:
\begin{align}
\eta(R) = \text{CH}\{ \eta(y_1) \oplus ... \oplus \eta(y_k) \,|\, R(y_1,...,y_k) \in \mathcal M\}
\end{align}
A proof that $\phi(\eta)= \mathcal M$, and that $\eta$ satisfies Conditions 1 and 2 from Definition~\ref{defGeometricModel}, is provided in the  appendix.

\smallskip For the next corollary we assume that the quasi-chained ontology $\Sigma$ belongs to  fragments enjoying
the \emph{finite model property (FMP)}, i.e.\    if a KB $\mathcal K$ is satisfiable, it has a finite model, e.g.\ where $\Sigma$ is  weakly-acyclic~\cite{FaginKMP05}, guarded, linear, or a quasi-chained datalog program. 
The following then is a direct consequence of Proposition  \ref{propmain}.

  
\begin{corollary}
Let $\mathcal K=(\Sigma, D)$ with $\Sigma$ as above. 
It holds that $\mathcal{K}$ is satisfiable iff $\mathcal{K}$ has a convex geometric model.
\end{corollary}

Intuitively, we require logics enjoying the FMP since the construction in the proof of Proposition 3 
  uses one dimension for each object that appears in a given model of the knowledge base. For ontologies expressed in fragments without the FMP, we can thus not guarantee the existence of an Euclidean model using this argument.

 \smallskip A natural question   is whether  there is a way of defining  a convex $n$-dimensional geometric model for an $n$ considerably smaller than $m =|\mathfrak{O}(\mathcal{M})|$ for some model $\mathcal{M}$. For the case of datalog rules, where $|\mathfrak{O}(\mathcal{M})| =|\mathbf{C}(\mathcal{K})|$, it turns out that this is in general not possible.

\begin{proposition}\label{propLowerBound}
For each $n\in \mathbb{N}$, there exists a KB $\mathcal K= (\Sigma, D)$ with  $\Sigma$   a datalog program,  over a signature with $n$ constants and $n$ unary predicates such that $\mathcal K$ does not have a convex geometric model in $\mathbb{R}^m$ for $m<n-1$.
\end{proposition}

To see this, consider the knowledge base $\mathcal K =(\Sigma,D)$ with $D = \{A_i(a_j) \mid  1 \leq i \neq j \leq n \}$,  for some $n \in \mathbb N$, and $\Sigma$ consisting of the following rule
\begin{align}\label{eqRuleDimProperty1A}
A_1(X) \wedge ... \wedge A_n(X) \rightarrow \bot
\end{align}
%

It is clear that $\mathcal K$ is satisfiable. Now, let $\eta$ be an $n-2$ dimensional convex geometric model of $\mathcal K$. Clearly, for each $a_j \in \mathbf{C}(\mathcal K)$, it holds that $\eta(a_j)\in \bigcap_{i\neq j} \eta(A_{i})$ and thus $\bigcap_{i\neq j} \eta(A_{i}) \neq \emptyset$. Using Helly's Theorem\footnote{This theorem states that if $C_1,...,C_k$ are convex regions in $\mathbb{R}^n$, with $k>n$, and each $n+1$ among these regions have a non-empty intersection, it holds that $\bigcap_{i=1}^k C_i \neq \emptyset$.}, it follows that $\bigcap_{i=1}^n  \eta(A_{i})$ contains some point $p$. Further, let $\eta^*$ be the extension of $\eta$ to $\mathbf{C}(\mathcal{K})\cup \{d\}$ defined by $\eta^*(d)=p$. Then $\mathcal K \cup \phi(\eta^*)$ contains $ D \cup \{A_i(d)\mid i \in [1,n]\}$ which together with \eqref{eqRuleDimProperty1A} implies  that $\mathcal K \cup \phi(\eta^*)$ does not have a convex model. Thus,  $\eta$ cannot be an $n-2$ dimensional convex geometric model $\mathcal K$, and  the dimensionality of any convex model of $\mathcal K$ has to be at least $n-1$.

Note that the model $\eta$ that we constructed above is $m$-dimensional, but the lower bound from Proposition \ref{propLowerBound} only states that at least $m-1$ dimensions are needed in general. In fact, it is easy to see that such an $m-1$-dimensional convex geometric model indeed exists for datalog programs. In particular, let $H$ be the hyperplane defined by $H = \{(p_1,...,p_m) \,|\, p_1 + ... + p_m = 1\}$ then clearly $\eta(x_i) \in H$ for every constant $x_i$ and $\eta(R) \subseteq H \oplus ... \oplus H$. In other words, each $\eta(x_i)$ is located in an $m-1$ dimensional space, and $\eta(R)$ is a subset of an $k \cdot (m-1)$ dimensional space.


\medskip \noindent {\bf Beyond Quasi-chained Rules.} The main remaining question  is whether the restriction to QC rules is necessary. The next example illustrates that if a KB contains rules that do not satisfy this restriction, it may not be possible to construct a convex geometric model. 
\begin{example}\label{counterExampleNonChainRules}
Consider $\Sigma$ consisting of the following rule:
\begin{align*}
R_1(X,Y) \wedge R_2(X,Y) \rightarrow \bot
\end{align*}
and let $D=\{R_1(a_1,a_1),\allowbreak R_1(a_2,a_2),\allowbreak R_2(a_1,a_2),\allowbreak R_2(a_2,a_1)\}$. Then clearly $\mathcal M =D$ is a model of the knowledge base $(\Sigma,D)$. Now suppose this KB had a convex geometric model $\eta$. Let $\eta^*$ be an extension of $\eta$ to the fresh constant $b$, defined by $\eta^*(b) = 0.5 \eta(a_1) + 0.5 \eta(a_2)$. Note that we then have:
\begin{align*}
\eta^*(b) \oplus \eta^*(b) &= 0.5(\eta(a_1) \oplus \eta(a_1) + 0.5(\eta(a_2) \oplus \eta(a_2)\\
&= 0.5(\eta(a_1) \oplus \eta(a_2) + 0.5(\eta(a_2) \oplus \eta(a_1)
\end{align*}
and thus, by the convexity of $\eta(R_1)$ and $\eta(R_2)$, it follows that $\eta^* \models R_1(b,b) \wedge R_2(b,b)$. This means that $(\Sigma, D\cup \phi(\eta^*))$ does not have a model, which contradicts the assumption that $\eta$ was a geometric model.
\end{example}

\section{Extended Geometric Models}\label{secExtended}

As shown in Section \ref{secMainResult}, there are knowledge bases which have a finite model but which do not have a convex geometric model. To deal with arbitrary knowledge bases, one possible approach is to simply drop the convexity requirement. In this section, we briefly explore another solution, based on the idea that for each relation symbol $R \in \mathfrak{R}_k(\mathcal{K})$, we can consider a function $f_R$ which embeds $k$-tuples into another vector space. This can be formalized as follows

\begin{definition}[Extended convex geometric interpretation]\label{defGeneralizedGeometricInterpretation}
Let $\mathfrak R$ be a set of relation names and ${\bf X}\subseteq {\bf C} \cup {\bf N}$ be a set of objects. An \emph{$m$-dimensional extended convex geometric interpretation of $(\mathfrak{R},{\bf X})$} is a pair $((f_R)_{R\in \mathfrak{R}},\eta)$, where for each $R\in \mathfrak{R}_k$, $f_R$ is a $\mathbb{R}^{k\cdot m} \rightarrow \mathbb{R}^{l_R}$ mapping, for some $l_R \in \mathbb{N}$, and $\eta$ assigns to each $R\in  \mathfrak{R}_k$ a convex region $\eta(R)$ in $\mathbb{R}^{l_R}$ and to each constant $c$ from ${\bf X}$ a vector $\eta(c) \in \mathbb{R}^m$.
\end{definition}

\noindent We can now adapt the definition of satisfaction of a ground atom as follows.

\begin{definition}[Satisfaction of ground atoms]
Let $((f_R)_{R\in \mathfrak{R}},\eta)$ be an extended convex geometric interpretation of $(\mathfrak{R},{\bf X})$, $R \in \mathfrak{R}_k$ and  $o_1,...,o_k\in {\bf X}$. We say that $\eta$ \emph{satisfies} a ground atom $R(o_1,...,o_k)$, written $\eta \models R(o_1,...,o_k)$, if $f_R(\eta(o_1) \oplus ... \oplus \eta(o_k)) \in \eta(R)$. 
\end{definition}

\noindent The notion of extended convex geometric model is then defined as in Definition \ref{defGeometricModel}, by simply using extended convex geometric models instead of (standard) geometric models. 

Note that we almost trivially have that every knowledge base  $\mathcal{K} = (\Sigma, D)$ which has a finite model $\mathcal M$ also has an extended convex geometric model. Indeed, to construct such a model, we can choose $\eta$ for constants  from $\bf X$ arbitrarily, as long as $\eta(o_1)\neq \eta(o_2)$ if $o_1\neq o_2$. We can then define $f_R$ as follows: $f_R(\mathbf{x})=1$ if $\mathcal M$ contains a ground atom $R(o_1,...,o_k)$ such that $\mathbf{x} = \eta(o_1)\oplus...\oplus \eta(o_k)$, and $f_R(\mathbf{x})=0$ otherwise. Finally we can define $\eta(R) = \{1\}$. It can be readily checked  that the extended convex geometric interpretation which is constructed in this way is indeed an extended convex geometric model of $\mathcal{K}$. 

The extended convex geometric model which is constructed in this way is uninteresting, however, as it does not allow us to use the geometric representations of the constants to induce any knowledge which is not already given in $\mathcal{K}$. Specifically, suppose $\mathbf{v_1},...,\mathbf{v_n} \in \mathbb{R}^m$ and let $\eta^*$ be the extension of $\eta$ to $\mathbf X\cup \{o_1,...,o_n\}$, then for $o'_1,...,o'_k \in \mathbf  X\cup \{o_1,...,o_n\}$ and $R\in\mathfrak{R}_k$, we have $\eta^* \models R(o'_1,..,o'_k)$ iff $\mathcal M$ contains some atom $R(p_1,...,p_k)$ such that $\eta(p_1)=\eta^*(o'_1),...,\eta(p_k)=\eta^*(o'_k)$. 
This means that in practice, we need to impose some restrictions on the functions $f_R$. Note, however, that we cannot restrict $f_R$ to be linear, as that would lead to the same restrictions as we encountered for standard convex geometric models. For instance, it is easy to verify that the knowledge base from Example \ref{counterExampleNonChainRules} cannot have an extended geometric model in which $f_{R_1}$ and $f_{R_2}$ are linear.

One possible alternative would be to encode each function $f_R$ as a neural network, but there are still several important open questions related to this choice. First, it is far from clear how we would then be able to check whether an extended convex geometric interpretation is a model of a given ontology. In contrast, for standard convex geometric interpretations, we can use standard linear programming techniques to check whether a given existential rule is satisfied. It is furthermore unclear which types of neural networks would be needed to guarantee that all types of existential rules can be captured. 

\section{Related Work}\label{secrelated}

Various  approaches to KG completion have been proposed that are based on neural network architectures \cite{NTN,Niepert16,MinerviniDRR17}. Interestingly, some of these approaches can be seen as special cases of the extended convex geometric models  considered in Section \ref{secExtended}. For example, in the E-MLP model \cite{NTN}, to predict whether $(e,R,f)$ is a valid triple, the concatenation of the vectors $\mathbf e$ and $\mathbf f$ is fed into a two-layer neural network.

Instead of constructing tuple representations from entity embeddings, some authors have also considered approaches that directly learn a vector space embedding of entity tuples \cite{Turney:2005:MSS:1642293.1642475,DBLP:conf/naacl/RiedelYMM13}. For each relation $R$ a vector $\mathbf{r}$ can then be learned such that the dot product $\mathbf{r} \cdot \mathbf{t}$ 
reflects the likelihood that a tuple represented by $\mathbf{t}$ is an instance of $R$. This model does not put any a priori restrictions on the kind of relations that can be modeled, although it is clearly not suitable for modelling rules (e.g.\ it is easy to see that this model carries over the limitations of bilinear models). Moreover, as enough information needs to be available about each tuple, this strategy has primarily been used for modelling knowledge extracted from text, where representations of word-tuples are learned from sentences that contain these words.

Note that KG embedding methods model relations in a soft way: their associated scoring function can be used to rank ground facts according to their likelihood of being correct, but no attempt is made at modelling the exact extension of relations. This means that logical dependencies among relations cannot be modeled, which makes such representations fundamentally different from the geometric representations that we have considered in this paper. Nonetheless, some authors have used logical rules to improve the predictions that are made in a KG completion setting. For example, in \cite{wang2015knowledge}, a mixed integer programming formulation is used to combine the predictions made from a given KG embedding with a set of hard rules. Specifically, the aim of this approach is to determine the most plausible set of facts which is logically consistent with the given rules. Another strategy, used in \cite{DBLP:conf/emnlp/DemeesterRR16}, is to incorporate background knowledge in the loss function of the learning problem. Specifically, the authors propose to take advantage of relation inclusions, i.e.\ rules of the form $R(X,Y) \rightarrow S(X,Y)$, for learning better tuple embeddings. The main underlying idea is to translate such a rule to the soft constraint that $\mathbf{r} \cdot \mathbf{t} \leq \mathbf{s} \cdot \mathbf{t}$
should hold for each tuple $t$. This is imposed in an efficient way by restricting tuple embeddings to vectors with non-negative coordinates and then requiring that $r_i \leq s_i$ for each coordinate $r_i$ of $\mathbf{r}$ and corresponding coordinate $s_i$ of $\mathbf{s}$. However, this strategy cannot straightforwardly be generalized to other types of rules. 

To overcome this shortcoming, neural network architectures dealing with arbitrary Datalog-like rules have been recently proposed~\cite{Niepert16,MinerviniDRR17}. Other related approaches include \cite{WangC16,DBLP:conf/nips/Rocktaschel017,DBLP:conf/ilp/SourekSZSK17}. However, such methods essentially use neural network methods to simulate deductive inference, but do not explicitly model the extension of relations, and do not allow for the tight integration of induction and deduction that our framework supports. Moreover, these methods are aimed at learning (soft versions of) first-order rules from data, rather then constraining embeddings based on a given set of (hard) rules.

Within KR research,  \cite{HoheneckerL17} recently made first steps towards the integration of ontological reasoning and deep learning, obtaining  encouraging results. Indeed, the developed system was considerably faster than the state of the art RDFox~\cite{NenovPMHWB15}, while retaining high-accuracy. Initial results have also been obtained in the use of ontological reasoning to derive human-interpretable explanations from  the output of a neural network~\cite{SarkerXDRH17}.




\section{Conclusions and Future Work}

We have argued that knowledge base embedding models should be capable of representing sufficiently expressive classes of rules, a property which, to the best of our knowledge, has not yet been considered in the literature. We found that the commonly used translation-based  and bilinear models are prohibitively restrictive in this respect. In light of this, we argue that more work is needed to better understand how different kinds of rules can be geometrically represented. In this paper, we have initiated this analysis, by studying
 knowledge base embeddings in which relations are represented as convex regions in a space of tuples. These tuples are simply represented as concatenations of the vector representations of the individual arguments, and can thus be obtained using standard approaches for learning entity embeddings. 

Our main finding is that using this convex-regions approach, knowledge bases that are restricted to the important class of quasi-chained existential rules can be faithfully encoded,
  in the sense that any set of facts which is induced using that vector space embedding is logically consistent and deductively closed with respect to the input ontology. Note that this is an essential requirement if we want to exploit symbolic knowledge when learning embeddings. For example, one common strategy is to encode (soft versions of) the given rules in the loss function, but for such a strategy to be successful, we should ensure that the considered representation is actually capable of satisfying the corresponding (soft) constraints.
We thus believe  this paper provides an important step towards a comprehensive  integration of neural embeddings and KR technologies,  laying important foundations to  develop methods that combine deductive and inductive reasoning in a tighter way than current approaches.

As future work, the most important next step is to develop practical region-based embedding models. 
Allowing arbitrary polytopes would likely lead to overfitting, but we believe that by appropriately restricting the types of regions that are allowed and regularizing the embedding model in an appropriate way, it will be possible to make more accurate predictions than existing knowledge graph embedding models. For example, note that translation based models, as well as bilinear models when restricted to positive coordinates, are special cases of region based models, so a natural approach would be to learn region based models that are regularized to stay close to these standard approaches. From a theoretical point of view, an important open problem is to characterize particular classes of extended convex geometric models that are sufficiently expressive to model arbitrary existential rules (or interesting sub-classes). Indeed, the non-linear representation from Section~\ref{secExtended} is too general to be practically useful, and  we therefore need to characterize what types of knowledge bases can be captured by different kinds of simple neural network architectures. 
Finally, it would be interesting to extend our framework to model  recently introduced ontology languages
especially tailored for KGs \cite{KrotzschMOT17}, which include means for representing  annotations on  data and relations.


\section*{Acknowledgments} The authors were supported by  EU's Horizon 2020 programme under the Marie Sk{\l}odowska-Curie grant 663830 and ERC Starting Grant 637277 `FLEXILOG', respectively.

\bibliography{commonsense,wordembedding,entity}

\begin{thebibliography}{}

\bibitem[\protect\citeauthoryear{Baader \bgroup et al\mbox.\egroup
  }{2017}]{DLBook}
Baader, F.; Horrocks, I.; Lutz, C.; and Sattler, U.
\newblock 2017.
\newblock {\em An Introduction to Description Logic}.
\newblock Cambridge University Press.

\bibitem[\protect\citeauthoryear{Bollacker \bgroup et al\mbox.\egroup
  }{2008}]{bollacker2008freebase}
Bollacker, K.; Evans, C.; Paritosh, P.; Sturge, T.; and Taylor, J.
\newblock 2008.
\newblock Freebase: a collaboratively created graph database for structuring
  human knowledge.
\newblock In {\em Proceedings of the ACM SIGMOD International Conference on
  Management of Data},  1247--1250.

\bibitem[\protect\citeauthoryear{Bordes \bgroup et al\mbox.\egroup
  }{2013}]{NIPS20135071}
Bordes, A.; Usunier, N.; Garcia-Duran, A.; Weston, J.; and Yakhnenko, O.
\newblock 2013.
\newblock Translating embeddings for modeling multi-relational data.
\newblock In {\em Proc.\ NIPS}.
\newblock  2787--2795.

\bibitem[\protect\citeauthoryear{Cal{\`{\i}}, Gottlob, and
  Kifer}{2013}]{CaliGK13}
Cal{\`{\i}}, A.; Gottlob, G.; and Kifer, M.
\newblock 2013.
\newblock Taming the infinite chase: Query answering under expressive
  relational constraints.
\newblock {\em J. Artif. Intell. Res.} 48:115--174.

\bibitem[\protect\citeauthoryear{Camacho{-}Collados, Pilehvar, and
  Navigli}{2016}]{DBLP:journals/ai/Camacho-Collados16}
Camacho{-}Collados, J.; Pilehvar, M.~T.; and Navigli, R.
\newblock 2016.
\newblock Nasari: Integrating explicit knowledge and corpus statistics for a
  multilingual representation of concepts and entities.
\newblock {\em Artificial Intelligence} 240:36--64.

\bibitem[\protect\citeauthoryear{Carlson \bgroup et al\mbox.\egroup
  }{2010}]{carlson2010toward}
Carlson, A.; Betteridge, J.; Kisiel, B.; Settles, B.; {Hruschka Jr.}, E.~R.;
  and Mitchell, T.~M.
\newblock 2010.
\newblock Toward an architecture for never-ending language learning.
\newblock In {\em Proc.\ AAAI},  1306--1313.

\bibitem[\protect\citeauthoryear{Demeester, Rockt{\"{a}}schel, and
  Riedel}{2016}]{DBLP:conf/emnlp/DemeesterRR16}
Demeester, T.; Rockt{\"{a}}schel, T.; and Riedel, S.
\newblock 2016.
\newblock Lifted rule injection for relation embeddings.
\newblock In {\em Proc.\ EMNLP},  1389--1399.

\bibitem[\protect\citeauthoryear{Dong \bgroup et al\mbox.\egroup
  }{2014}]{dong2014knowledge}
Dong, X.; Gabrilovich, E.; Heitz, G.; Horn, W.; Lao, N.; Murphy, K.; Strohmann,
  T.; Sun, S.; and Zhang, W.
\newblock 2014.
\newblock Knowledge vault: A web-scale approach to probabilistic knowledge
  fusion.
\newblock In {\em SIGKDD},  601--610.

\bibitem[\protect\citeauthoryear{Fagin \bgroup et al\mbox.\egroup
  }{2005}]{FaginKMP05}
Fagin, R.; Kolaitis, P.~G.; Miller, R.~J.; and Popa, L.
\newblock 2005.
\newblock Data exchange: semantics and query answering.
\newblock {\em Theor. Comput. Sci.} 336(1):89--124.

\bibitem[\protect\citeauthoryear{{G\"ardenfors}}{2000}]{Gardenfors:conceptualSpaces}
{G\"ardenfors}, P.
\newblock 2000.
\newblock {\em Conceptual Spaces: The Geometry of Thought}.
\newblock MIT Press.

\bibitem[\protect\citeauthoryear{Gardner and Mitchell}{2015}]{GardnerM15}
Gardner, M., and Mitchell, T.~M.
\newblock 2015.
\newblock Efficient and expressive knowledge base completion using subgraph
  feature extraction.
\newblock In {\em Proc.\ of {EMNLP}-15},  1488--1498.

\bibitem[\protect\citeauthoryear{Hohenecker and
  Lukasiewicz}{2017}]{HoheneckerL17}
Hohenecker, P., and Lukasiewicz, T.
\newblock 2017.
\newblock Deep learning for ontology reasoning.
\newblock {\em arXiv preprint arxiv:1705.10342}.

\bibitem[\protect\citeauthoryear{Kazemi and Poole}{2018}]{kazemi2018simple}
Kazemi, S.~M., and Poole, D.
\newblock 2018.
\newblock {SimplE} embedding for link prediction in knowledge graphs.
\newblock {\em arXiv preprint arXiv:1802.04868}.

\bibitem[\protect\citeauthoryear{Kr{\"{o}}tzsch \bgroup et al\mbox.\egroup
  }{2017}]{KrotzschMOT17}
Kr{\"{o}}tzsch, M.; Marx, M.; Ozaki, A.; and Thost, V.
\newblock 2017.
\newblock Attributed description logics: Ontologies for knowledge graphs.
\newblock In {\em Proc.\ of {ISWC}-17}.

\bibitem[\protect\citeauthoryear{Lin \bgroup et al\mbox.\egroup
  }{2015}]{TransR}
Lin, Y.; Liu, Z.; Sun, M.; Liu, Y.; and Zhu, X.
\newblock 2015.
\newblock Learning entity and relation embeddings for knowledge graph
  completion.
\newblock In {\em AAAI},  2181--2187.

\bibitem[\protect\citeauthoryear{Miller}{1995}]{miller1995wordnet}
Miller, G.~A.
\newblock 1995.
\newblock Wordnet: a lexical database for english.
\newblock {\em Communications of the ACM} 38:39--41.

\bibitem[\protect\citeauthoryear{Minervini \bgroup et al\mbox.\egroup
  }{2017}]{MinerviniDRR17}
Minervini, P.; Demeester, T.; Rockt{\"{a}}schel, T.; and Riedel, S.
\newblock 2017.
\newblock Adversarial sets for regularising neural link predictors.
\newblock In {\em Proc.\ of {UAI}-17}.

\bibitem[\protect\citeauthoryear{Nenov \bgroup et al\mbox.\egroup
  }{2015}]{NenovPMHWB15}
Nenov, Y.; Piro, R.; Motik, B.; Horrocks, I.; Wu, Z.; and Banerjee, J.
\newblock 2015.
\newblock {RDFox}: {A} highly-scalable {RDF} store.
\newblock In {\em Proc.\ of {ISWC}-15},  3--20.

\bibitem[\protect\citeauthoryear{Nguyen \bgroup et al\mbox.\egroup
  }{2016}]{STransE}
Nguyen, D.~Q.; Sirts, K.; Qu, L.; and Johnson, M.
\newblock 2016.
\newblock {STransE}: a novel embedding model of entities and relationships in
  knowledge bases.
\newblock In {\em Proc.\ of NAACL-HLT},  460--466.

\bibitem[\protect\citeauthoryear{Nickel, Tresp, and
  Kriegel}{2011}]{nickel2011three}
Nickel, M.; Tresp, V.; and Kriegel, H.-P.
\newblock 2011.
\newblock A three-way model for collective learning on multi-relational data.
\newblock In {\em Proc.\ ICML},  809--816.

\bibitem[\protect\citeauthoryear{Niepert}{2016}]{Niepert16}
Niepert, M.
\newblock 2016.
\newblock Discriminative gaifman models.
\newblock In {\em Proc.\ of {NIPS}-16},  3405--3413.

\bibitem[\protect\citeauthoryear{Riedel \bgroup et al\mbox.\egroup
  }{2013}]{DBLP:conf/naacl/RiedelYMM13}
Riedel, S.; Yao, L.; McCallum, A.; and Marlin, B.~M.
\newblock 2013.
\newblock Relation extraction with matrix factorization and universal schemas.
\newblock In {\em Proc. HLT-NAACL},  74--84.

\bibitem[\protect\citeauthoryear{Rockt{\"{a}}schel and
  Riedel}{2017}]{DBLP:conf/nips/Rocktaschel017}
Rockt{\"{a}}schel, T., and Riedel, S.
\newblock 2017.
\newblock End-to-end differentiable proving.
\newblock In {\em Proc.\ NIPS},  3791--3803.

\bibitem[\protect\citeauthoryear{Rosseel}{2002}]{ROSSEEL2002178}
Rosseel, Y.
\newblock 2002.
\newblock Mixture models of categorization.
\newblock {\em Journal of Mathematical Psychology} 46(2):178 -- 210.

\bibitem[\protect\citeauthoryear{Sarker \bgroup et al\mbox.\egroup
  }{2017}]{SarkerXDRH17}
Sarker, M.~K.; Xie, N.; Doran, D.; Raymer, M.; and Hitzler, P.
\newblock 2017.
\newblock Explaining trained neural networks with semantic web technologies:
  First steps.
\newblock In {\em Proc.\ of {NeSy}-17}.

\bibitem[\protect\citeauthoryear{Shmueli}{1987}]{Shmueli87}
Shmueli, O.
\newblock 1987.
\newblock Decidability and expressiveness of logic queries.
\newblock In {\em Proc.\ of {PODS}-87},  237--249.

\bibitem[\protect\citeauthoryear{Socher \bgroup et al\mbox.\egroup
  }{2013}]{NTN}
Socher, R.; Chen, D.; Manning, C.~D.; and Ng, A.
\newblock 2013.
\newblock Reasoning with neural tensor networks for knowledge base completion.
\newblock In {\em Proc.\ NIPS},  926--934.

\bibitem[\protect\citeauthoryear{Sourek \bgroup et al\mbox.\egroup
  }{2017}]{DBLP:conf/ilp/SourekSZSK17}
Sourek, G.; Svatos, M.; Zelezn{\'{y}}, F.; Schockaert, S.; and Kuzelka, O.
\newblock 2017.
\newblock Stacked structure learning for lifted relational neural networks.
\newblock In {\em Proc.\ ILP},  140--151.

\bibitem[\protect\citeauthoryear{Speer, Chin, and
  Havasi}{2017}]{speer2017conceptnet}
Speer, R.; Chin, J.; and Havasi, C.
\newblock 2017.
\newblock Conceptnet 5.5: An open multilingual graph of general knowledge.
\newblock In {\em Proc.\ AAAI},  4444--4451.

\bibitem[\protect\citeauthoryear{Toutanova \bgroup et al\mbox.\egroup
  }{2015}]{toutanova2015representing}
Toutanova, K.; Chen, D.; Pantel, P.; Poon, H.; Choudhury, P.; and Gamon, M.
\newblock 2015.
\newblock Representing text for joint embedding of text and knowledge bases.
\newblock In {\em Proc.\ of {EMNLP}-15},  1499--1509.

\bibitem[\protect\citeauthoryear{Trouillon \bgroup et al\mbox.\egroup
  }{2016}]{ComplEx}
Trouillon, T.; Welbl, J.; Riedel, S.; Gaussier, {\'{E}}.; and Bouchard, G.
\newblock 2016.
\newblock Complex embeddings for simple link prediction.
\newblock In {\em Proc.\ ICML},  2071--2080.

\bibitem[\protect\citeauthoryear{Turney}{2005}]{Turney:2005:MSS:1642293.1642475}
Turney, P.~D.
\newblock 2005.
\newblock Measuring semantic similarity by latent relational analysis.
\newblock In {\em Proc.\ IJCAI},  1136--1141.

\bibitem[\protect\citeauthoryear{Ullman and Gelder}{1988}]{UllmanG88}
Ullman, J.~D., and Gelder, A.~V.
\newblock 1988.
\newblock Parallel complexity of logical query programs.
\newblock {\em Algorithmica} 3:5--42.

\bibitem[\protect\citeauthoryear{Vrande{\v{c}}i{\'c} and
  Kr{\"o}tzsch}{2014}]{vrandevcic2014wikidata}
Vrande{\v{c}}i{\'c}, D., and Kr{\"o}tzsch, M.
\newblock 2014.
\newblock Wikidata: a free collaborative knowledge base.
\newblock {\em Communications of the ACM} 57:78--85.

\bibitem[\protect\citeauthoryear{Wang and Cohen}{2016}]{WangC16}
Wang, W.~Y., and Cohen, W.~W.
\newblock 2016.
\newblock Learning first-order logic embeddings via matrix factorization.
\newblock In {\em Proc.\ of {IJCAI}-16},  2132--2138.

\bibitem[\protect\citeauthoryear{Wang \bgroup et al\mbox.\egroup
  }{2014}]{TransH}
Wang, Z.; Zhang, J.; Feng, J.; and Chen, Z.
\newblock 2014.
\newblock Knowledge graph embedding by translating on hyperplanes.
\newblock In {\em AAAI},  1112--1119.

\bibitem[\protect\citeauthoryear{Wang \bgroup et al\mbox.\egroup
  }{2017}]{WangMWG17}
Wang, Q.; Mao, Z.; Wang, B.; and Guo, L.
\newblock 2017.
\newblock Knowledge graph embedding: {A} survey of approaches and applications.
\newblock {\em {IEEE} Trans. Knowl. Data Eng.} 29(12):2724--2743.

\bibitem[\protect\citeauthoryear{Wang, Wang, and Guo}{2015}]{wang2015knowledge}
Wang, Q.; Wang, B.; and Guo, L.
\newblock 2015.
\newblock Knowledge base completion using embeddings and rules.
\newblock In {\em Proc.\ IJCAI},  1859--1866.

\bibitem[\protect\citeauthoryear{Xiao \bgroup et al\mbox.\egroup }{2017}]{SSP}
Xiao, H.; Huang, M.; Meng, L.; and Zhu, X.
\newblock 2017.
\newblock Ssp: Semantic space projection for knowledge graph embedding with
  text descriptions.
\newblock In {\em Proc.\ AAAI}, volume~17,  3104--3110.

\bibitem[\protect\citeauthoryear{Xie \bgroup et al\mbox.\egroup
  }{2016}]{xie2016representation}
Xie, R.; Liu, Z.; Jia, J.; Luan, H.; and Sun, M.
\newblock 2016.
\newblock Representation learning of knowledge graphs with entity descriptions.
\newblock In {\em Proc.\ of AAAI},  2659--2665.

\bibitem[\protect\citeauthoryear{Yang \bgroup et al\mbox.\egroup
  }{2015}]{yang2014embedding}
Yang, B.; Yih, W.; He, X.; Gao, J.; and Deng, L.
\newblock 2015.
\newblock Embedding entities and relations for learning and inference in
  knowledge bases.
\newblock In {\em Proc.\ of {ICLR}-15}.

\bibitem[\protect\citeauthoryear{Zhong \bgroup et al\mbox.\egroup
  }{2015}]{DBLP:conf/emnlp/ZhongZWWC15}
Zhong, H.; Zhang, J.; Wang, Z.; Wan, H.; and Chen, Z.
\newblock 2015.
\newblock Aligning knowledge and text embeddings by entity descriptions.
\newblock In {\em EMNLP},  267--272.

\end{thebibliography}
\bibliographystyle{aaai}

\clearpage

\ifappendix
\appendix
\section{APPENDIX}

\medskip
\subsection{Proof of Proposition \ref{propLimitationBilinear}}

Let us write $r_{ij}$ for the element on the $i^{\textit{th}}$ row and $j^{\textit{th}}$ column of $M_r$, and similar for $M_s$.

\begin{lemma}\label{lemmaBilinear1}
Suppose $M_r$ and $M_s$ are matrices for which \eqref{eqArg2RuleBilinear} is satisfied. Let $k,l \in \{1,...,n\}$ be such that $s_{kl} \neq 0$. For each $m\in \{1,...,n\}$ such that $s_{km}=0$ it holds that $r_{km}=0$.
\end{lemma}
\begin{proof}
Assume that there exists some index $m$ such that $s_{km}=0$ but $r_{km}\neq 0$. Then we show that \eqref{eqArg2RuleBilinear} cannot be satisfied. We define $\mathbf{e} = (e_1,...,e_n)$ and $\mathbf{f} = (f_1,...,f_n)$ as follows:
\begin{align*}
e_i &= 
\begin{cases}
0 & \text{if $i\neq k$}\\
1 & \text{otherwise}
\end{cases}&
f_i &= 
\begin{cases}
0 & \text{if $i\notin \{l,m\}$}\\
\frac{K}{r_{kl}} & \text{if $i=l$}\\
\frac{L}{s_{km}} & \text{if $i=m$}
\end{cases}
\end{align*}
Then we have:
\begin{align*}
\mathbf{e}^T  M_r  \mathbf{f} \geq \lambda_r &= K + \frac{L \cdot r_{km}}{s_{km}}\\
\mathbf{e}^T M_s \mathbf{f}  &= L
\end{align*}
We can choose for $L$ an arbitrary value such that $L< \lambda_s$, and then choose $K$ such that $K + \frac{L \cdot r_{km}}{s_{km}} \geq \lambda_r$. It follows that \eqref{eqArg2RuleBilinear} is not satisfied for $M_r$ and $M_s$.
\end{proof}

\begin{lemma}\label{lemmaBilinear2}
Suppose $M_r$ and $M_s$ are matrices for which \eqref{eqArg2RuleBilinear} is satisfied. Let $k,l \in \{1,...,n\}$ be such that $s_{kl} \neq 0$. For each $m\in \{1,...,n\}$ such that $s_{ml}=0$ it holds that $r_{ml}=0$.
\end{lemma}
\begin{proof}
Entirely analogous to the proof of Lemma \ref{lemmaBilinear1}.
\end{proof}

\begin{lemma}\label{lemmaBilinearA}
Suppose the indices $1 \leq k,l,m \leq$ are such that $(r_{kl},r_{km}) \neq (\alpha s_{kl},\alpha s_{km})$ for all $\alpha \in \mathbb{R}$, and assume $s_{kl},s_{km}\neq 0$. Then it holds that $M_r$ and $M_s$ cannot satisfy \eqref{eqArg2RuleBilinear}.
\end{lemma}
\begin{proof}
Note that the assumptions imply that either $r_{kl}\neq 0$ or $r_{km}\neq 0$. Let us assume for instance that $r_{kl}\neq 0$; the case where $r_{km}\neq 0$ is entirely analogous.
We define $\mathbf{e} = (e_1,...,e_n)$ and $\mathbf{f} = (f_1,...,f_n)$ as follows:
\begin{align*}
e_i &= 
\begin{cases}
0 & \text{if $i\neq k$}\\
1 & \text{otherwise}
\end{cases}&
f_i &= 
\begin{cases}
0 & \text{if $i\notin \{l,m\}$}\\
\frac{K}{s_{kl}} & \text{if $i=l$}\\
\frac{L}{s_{km}} & \text{if $i=m$}
\end{cases}
\end{align*}
Then we have:
\begin{align*}
\mathbf{e}^T  M_r  \mathbf{f} &= K \cdot \frac{r_{kl}}{s_{kl}}+ L \cdot \frac{r_{km}}{s_{km}}\\
\mathbf{e}^T M_s \mathbf{f} &= K + L 
\end{align*}
Using the assumption we made that $r_{kl}\neq 0$ we can choose
$$
K = \lambda_r \frac{s_{kl}}{r_{kl}}   - L\cdot \frac{r_{km} s_{kl}}{s_{km} r_{kl}}
$$
which guarantees $\mathbf{e}^T  M_r  \mathbf{f}  = \lambda_r$. To guarantee that $\mathbf{e}^T M_s \mathbf{f} < \lambda_s$, for this particular choice of $K$, we need to ensure:
\begin{align*}
\lambda_s > \lambda_r \frac{s_{kl}}{r_{kl}}   - L\cdot \frac{r_{km} s_{kl}}{s_{km} r_{kl}}   + L
\end{align*}
Noting that $\frac{r_{km} s_{kl}}{s_{km} r_{kl}}\neq 1$ since we assumed $\frac{r_{kl}}{s_{kl}} \neq \frac{r_{km}}{s_{km}}$, this is either equivalent to one of
\begin{align*}
L  &< \frac{\lambda_s - \lambda_r \frac{s_{kl}}{r_{kl}} }{1 - \frac{r_{km} s_{kl}}{s_{km} r_{kl}}}
&L  &> \frac{\lambda_s - \lambda_r \frac{s_{kl}}{r_{kl}} }{1 - \frac{r_{km} s_{kl}}{s_{km} r_{kl}}}
\end{align*}
depending on the sign of $1 - \frac{r_{km} s_{kl}}{s_{km} r_{kl}}$. In particular, it follows that we can always ensure $\mathbf{e}^T M_s \mathbf{f} < \lambda_s$ by either choosing $L$ to be sufficiently small or sufficiently large.
\end{proof}

\begin{lemma}
Suppose the indices $1 \leq k,l,m \leq$ are such that $(r_{lk},r_{mk}) \neq (\alpha s_{lk},\alpha s_{mk})$ for all $\alpha \in \mathbb{R}$, and assume $s_{lk},s_{mk}\neq 0$. Then it holds that $M_r$ and $M_s$ cannot satisfy \eqref{eqArg2RuleBilinear}.
\end{lemma}
\begin{proof}
The proof is entirely analogous to the proof of Lemma \ref{lemmaBilinearA}, starting instead with the following choice for $\mathbf{e}$ and $\mathbf{f}$:
\begin{align*}
e_i &= 
\begin{cases}
0 & \text{if $i\notin \{l,m\}$}\\
\frac{K}{s_{lk}} & \text{if $i=l$}\\
\frac{L}{s_{mk}} & \text{if $i=m$}
\end{cases}&
f_i &= 
\begin{cases}
0 & \text{if $i\neq k$}\\
1 & \text{otherwise}
\end{cases}
\end{align*}
\end{proof}

\noindent From these results, it follows that when $M_r$ and $M_s$ satisfy \eqref{eqArg2RuleBilinear}, it has to be the case that $M_r = \alpha M_s$ for some $\alpha \in \mathbb{R}$. Moreover, it clearly has to be the case that $\alpha\geq 0$.
\subsection{Proof of Proposition \ref{propLimitationSimplE}}

Let us write $\mathbf{r}=(r_1,...,r_n)$ and $\mathbf{ri}=(ri_1,...,ri_n)$, and similar for $s$ and $t$. Let $\textit{sg}$ be the variant of the sign function defined by $\textit{sg}(x)=1$ if $x\geq 0$ and $\textit{sg}(x)=-1$ otherwise.

First note that if $\mathbf{t}=\mathbf{ti}=\mathbf{0}$ then it must be the case that $\lambda_t\leq 0$ and $\lambda_{ti}\leq 0$ and thus that \eqref{eqPropSimlpECompositionRule2} is satisfied. Let us therefore assume that $\mathbf{t}\neq \mathbf{0}$; the case where $\mathbf{ti}\neq \mathbf{0}$ is entirely analogous. We show that we can always find entity vectors for which the body of \eqref{eqTripleImplicationCompositionRule} is satisfied, while $\langle\mathbf{e_h}, \mathbf{t}, \mathbf{g_t} \rangle < \lambda_t$. 

First, let us define $\mathbf{e_h}=(1,....,1)$ and let $\mathbf{g_t}=(-K \cdot \textit{sg}(t_1),....,-K \cdot \textit{sg}(t_n))$, for some $K>0$. By choosing $K$ sufficiently large, it is clear that we can always ensure that $\langle\mathbf{e_h}, \mathbf{t}, \mathbf{g_t} \rangle < \lambda_t$, given that we assumed $\textbf{t}\neq \textbf{0}$. 

In a similar way, we show that the body of \eqref{eqTripleImplicationCompositionRule} can be satisfied. In particular, by defining $\mathbf{f_t} = (K \cdot \textit{sg}(r_1),....,K \cdot \textit{sg}(r_n))$, for a sufficiently large $K$ we have that $\langle\mathbf{e_h}, \mathbf{r}, \mathbf{f_t} \rangle \geq \lambda_r$, provided that $\mathbf{f_t}\neq \mathbf{0}$. If $\mathbf{f_t}= \mathbf{0}$ then we either have that $\langle\mathbf{e_h}, \mathbf{r}, \mathbf{f_t} \rangle \geq \lambda_r$ is satisfied for all choices of the vectors $\mathbf{e_h}$ and $\mathbf{f_t}$ (namely, if $\lambda_r \leq 0$ ) or for no choices. In the former case, we find again that $\langle\mathbf{e_h}, \mathbf{r}, \mathbf{f_t} \rangle \geq \lambda_r$ is satisfied. In the latter case we find that \eqref{eqPropSimlpECompositionRule1} is satisfied.

The remaining vectors are defined as follows:
\begin{align*}
\mathbf{g_h} = (&K \cdot \textit{sg}(r_1)\cdot \textit{sg}(si_1),....,K \cdot \textit{sg}(r_n) \cdot \textit{sg}(si_n)\\
\mathbf{f_h} = (&-K \cdot \textit{sg}(t_1)\cdot \textit{sg}(s_1),....,-K \cdot \textit{sg}(t_n) \cdot \textit{sg}(s_n))\\
\mathbf{e_t} = (&-K \cdot \textit{sg}(t_1)\cdot \textit{sg}(s_1)\cdot \textit{sg}(ri_1),...\\
&\quad\quad\quad\quad\quad\quad....,-K \cdot \textit{sg}(t_n) \cdot \textit{sg}(s_n)\cdot \textit{sg}(ri_n))
\end{align*}
We then clearly have that either \eqref{eqPropSimlpECompositionRule1} is satisfied or that $\langle\mathbf{g_h}, \mathbf{si}, \mathbf{f_t} \rangle \geq \lambda_{si}$, $\langle\mathbf{f_h}, \mathbf{s}, \mathbf{g_t} \rangle \geq \lambda_s$ and $\langle\mathbf{f_h}, \mathbf{ri}, \mathbf{e_t} \rangle \geq \lambda_{ri}$ are satisfied. Note that this would not hold for the standard sign function, as e.g.\ we might have that $\textit{sgn}(r_1)\cdot \textit{sgn}(si_1) = ... = \textit{sgn}(r_n) \cdot \textit{sgn}(si_n) = 0$.

\subsection{Proof of Proposition \ref{propmain}}

\begin{lemma}
Let $\eta$ be the geometric interpretation that was constructed from $\mathcal{M}$, as explained in the main paper. It holds that $\mathcal{M}=\phi(\eta)$.
\end{lemma}
\begin{proof}
Clearly, by definition of $\eta$, if $R(x_{i_1},...,x_{i_k}) \in \mathcal{M}$ for some $R\in\mathfrak{R}_k(\mathcal{K})$, we have $\eta(x_{i_1}) \oplus ... \oplus \eta(x_{i_k}) \in \eta(R)$ and thus $R(x_{i_1},...,x_{i_k}) \in \phi(\eta)$. Conversely, assume that $R(x_{i_1},...,x_{i_k}) \in \phi(\eta)$. This means that $\eta(x_{i_1}) \oplus ... \oplus \eta(x_{i_k})$ is in the convex hull of some vectors $\eta(x_{j_{1,1}}) \oplus ... \oplus \eta(x_{j_{k,1}}),...,\eta(x_{j_{1,p}}) \oplus ... \oplus \eta(x_{j_{k,p}})$ such that $R(x_{j_{1,1}},...,x_{j_{k,1}}),...,R(x_{j_{1,p}},...,x_{j_{k,p}})$ are all in $\mathcal{M}$. Note that $\eta(x_{i_1}) \oplus ... \oplus \eta(x_{i_k})$ has coordinate 1 at indices $i_1, m+i_2,..., (k-1)m + i_k$. It is easy to verify that having a coordinate of 1 at index $(l-1)m + i_l$ is only possible if $x_{i_l} = x_{j_{l,1}} = ... = x_{j_{l,p}}$. Since this holds for all $l$, it follows that $(x_{i_1},...,x_{i_k}) = (x_{j_{1,1}},...,x_{j_{k,1}}) = ... = (x_{j_{1,p}},...,x_{j_{k,p}})$, and thus that $R(x_{i_1},...,x_{i_k})\in \mathcal{M}$.
\end{proof}

We now show that the remaining conditions from Definition \ref{defGeometricModel} are indeed satisfied. Let $\{{\bf v_1},...,{\bf v_n}\} \subseteq \mathbb{R}^m$,  $a_1,...,a_n$ be fresh constants and  $\eta^*$ be the extension of $\eta$ to $ \mathfrak O(\mathcal M) \cup \{a_1,...,a_n\}$ defined by $\eta^*(a_i)=\mathbf v_i$. 

We first consider the case where $\Sigma$ is a datalog program. In this case, it is sufficient to show that $\mathcal{M}^* = \phi(\eta^*)$ is a model of any rule in $\Sigma$. We first show this for rules that have a single atom in their body.

\begin{lemma}\label{lemmaOneAtomBody}
Let $\mathcal{M}^*$ be defined as above, and let $\sigma \in \Sigma$ be a datalog rule of the following form:
$$
R(A_1,...,A_k) \rightarrow S(A_{s_1},...,A_{s_l}) 
$$
where $s_1,...,s_l$ are (not necessarily distinct) indices from $\{1,...,k\}$. Then $\mathcal{M}^*$ is a model of $\sigma$.
\end{lemma}
\begin{proof}
Suppose $R(z_1,...,z_k) \in \mathcal{M}^*$. Then by the definition of $\mathcal{M}^*$, it holds that 
$$
\eta^*(z_1)\oplus ... \oplus \eta^*(z_k) \in \eta^*(R)
$$
This means that there exist instances $(a_1^1,...,a_k^1),\allowbreak...,\allowbreak(a_1^q,...,a_k^q)$ of $R$, with each $a_i^j$ from $\mathfrak{O}(\mathcal{M})$, such that for $j\in \{1,...,k\}$ it holds that
\begin{align*}
\eta^*(z_j) = \lambda_1 \eta^*(a_j^1) + ... +  \lambda_q \eta^*(a_j^q)
\end{align*}
where $\lambda_1,...,\lambda_q > 0$ and $\sum_i \lambda_i=1$. Since $\mathcal{M}=\phi(\eta)$ is a model of $\sigma$ it must hold that $(a_{s_1}^1,...,a_{s_l}^1),...,(a_{s_1}^q,...,a_{s_l}^q)$ are all instances of $S$, and thus for each $i\in\{1,...,q\}$ it holds that
$$
\eta^*(a_{s_1}^i)\oplus ... \oplus \eta^*(a_{s_l}^i) \in \eta^*(S)
$$
Since $\eta^*(S)$ is convex, we thus also have that
$$
(\sum_i \lambda_i \eta^*(a_{s_1}^i)\oplus ... \oplus \eta^*(a_{s_l}^i)) \in \eta^*(S)
$$
or equivalently
$$
\eta^*(z_{s_1})\oplus ... \oplus \eta^*(z_{s_l}) \in \eta^*(S)
$$
which is what we needed to show.
\end{proof}

Below we show the analogue of Lemma \ref{lemmaOneAtomBody} for rules with two atoms in their body. However, we first need to show a technical lemma.
For $A \subseteq \mathbb{R}^p$ and $B \subseteq \mathbb{R}^l$, we define the set $A\oplus B \subseteq \mathbb{R}^{p+l}$ as 
$A \oplus B= \{ \mathbf{a} \oplus \mathbf{b} \,|\,  \mathbf{a}\in A, \mathbf{b}\in B\}$, i.e.\ it is the set of all concatenations of elements from $A$ with elements from $B$.
\begin{lemma}\label{lemmaRelationComposition}
Let $\eta^*$ and $\mathcal{M}^*$ be defined as before. Let $R\in \mathfrak{R}_k(\mathcal{K})$ and $S\in\mathfrak{R}_l(\mathcal{K})$.  Let $z_1,...,z_{k+l-1} \in \mathfrak{O}(\mathcal{M}) \cup \{a_1,...,a_n\}$. We have that $R(z_1,...,z_k)\in \mathcal{M}^*$ and $S(z_k,z_{k+1},...,z_{k+l-1}) \in \mathcal{M}^*$ iff $\eta^*(z_1) \oplus ... \oplus \eta^*(z_{k+l-1})$ belongs to
\begin{align}\label{eqLemmaRelationCompositionA}
(\eta^*(R) \oplus \mathbb{R}^{m(l-1)}) \cap (\mathbb{R}^{m(k-1)} \oplus \eta^*(S)) 
\end{align}
\end{lemma}
\begin{proof}
Clearly if \eqref{eqLemmaRelationCompositionA} holds, we have $\eta^*(z_1) \oplus ... \oplus\eta^*(z_k) \in \eta^*(R)$ and $\eta^*(z_k) \oplus ... \oplus\eta^*(z_{k+l-1}) \in \eta^*(S)$, and thus by definition of $\mathcal{M}^*$ we have that $R(z_1,...,z_k)\in \mathcal{M}^*$ and $S(z_k,z_{k+1},...,z_{k+l-1}) \in \mathcal{M}^*$.

Now conversely assume that $R(z_1,...,z_k)\in \mathcal{M}^*$ and $S(z_k,z_{k+1},...,z_{k+l-1}) \in \mathcal{M}^*$, or in other words that $\eta^*(z_1) \oplus ... \oplus\eta^*(z_k) \in \eta^*(R)$ and $\eta^*(z_k) \oplus ... \oplus\eta^*(z_{k+l-1}) \in \eta^*(S)$. That means that there are instances $(a_1^1,...,a_k^1),...,(a_1^q,...,a_k^q)$ of $R$ and instances $(b_k^1,...,b_{k+l-1}^1),...,(b_{k}^r,...,b_{k+l-1}^r)$ of $S$, with each $a_i^j$ and $b_i^j$ from $\mathfrak{O}(\mathcal{M})$, such that for $j\in \{1,...,k\}$ it holds that
\begin{align*}
\eta^*(z_j) = \lambda_1 \eta^*(a_j^1) + ... +  \lambda_q \eta^*(a_j^q)
\end{align*}
and for $j\in \{k,...,k+l-1\}$ it holds that
\begin{align*}
\eta^*(z_j) = \mu_1 \eta^*(b_j^1) + ... +  \mu_r \eta^*(b_j^r)
\end{align*}
where $\lambda_1,...,\lambda_q > 0$, $\mu_1,...,\mu_r > 0$, $\sum_i \lambda_i=1$ and $\sum_i \mu_i=1$. Note that this means $\lambda_1 \eta^*(a_k^1) + ... +  \lambda_q \eta^*(a_k^q) = \mu_1 \eta^*(b_k^1) + ... +  \mu_r \eta^*(b_k^r)$, which is only possible if $\{a_k^1,...,a_k^q\} = \{b_k^1,...,b_k^r\}$. It follows that $\eta^*(z_1) \oplus ... \oplus\eta^*(z_{k+l-1})$ can be written as:
\begin{align*}
&\sum \{\lambda(a_k^i) \mu(b_k^j) \big( \eta^*(a_1^i) \oplus ...\oplus \eta^*(a_k^i) \oplus \eta^*(b_{k+1}^j) \\
&\quad\quad \oplus ... \oplus \eta^*(b_{k+l-1}^j) \big)\,|\, 1\leq i \leq q, 1\leq j\leq r, a_k^i = b_k^j \}
\end{align*} 
where
\begin{align*}
\lambda(a_k^i) &= \sum_j \{\lambda_j \,|\, a_k^j=a_k^i\} \\
\mu(b_k^i) &= \sum_j \{\mu_j \,|\, b_k^j=b_k^i\}
\end{align*}
Note first of all that this is a convex combination, i.e.\ the weights $\lambda(a_k^i) \mu(b_k^j)$ for all $i$ and $j$ such that $a_k^i = b_k^j$ sum to 1. Furthermore, if $a_k^i=b_k^j$, we also have that 
$$
\eta^*(a_1^i) \oplus ... \oplus \eta^*(a_k^i) \oplus \eta^*(b_{k+1}^j) \oplus ... \oplus \eta^*(b_{k+l-1}^j) 
$$
belongs to the set $(\eta^*(R) \oplus \mathbb{R}^{l-1}) \cap (\mathbb{R}^{k-1} \oplus \eta^*(S)$, from which it follows that \eqref{eqLemmaRelationCompositionA} has to hold.
\end{proof}
The previous lemma essentially tells us that the tuples $(z_1,...,z_{k+l-1})$ that satisfy the body of a rule of the form $R(X_1,...,X_k) \wedge S(X_k,,...,X_{k+l-1})$ are characterized by a particular convex region. When there are two atoms in the body that share more than one variable, however, this property no longer holds, which is intuitively why we need to require quasi-chainedness. Thanks to Lemma \ref{lemmaRelationComposition} we can now easily show the following analogue of Lemma \ref{lemmaOneAtomBody} for rules with two atoms in their body.

\begin{lemma}\label{lemmaRulesTwoBodyPredicate}
Let $\mathcal{M}^*$ be defined as above, and let $\sigma \in \Sigma$ be a datalog rule of the following form:
$$
R_1(A_1,...,A_k) \wedge R_2 (A_k,...,A_{k+l-1}) \rightarrow  R_3(A_{s_1},...,A_{s_q}) 
$$
where $s_1,...,s_q$ are (not necessarily distinct) indices from $\{1,...,k+l-1\}$. Then $\mathcal{M}^*$ is a model of $\sigma$.
\end{lemma}
\begin{proof}
The proof is similar to the proof of Lemma \ref{lemmaOneAtomBody}.
In particular, suppose $R_1(z_1,...,z_k) \in M^*$ and $R_2(z_k,...,z_{k+l-1})\in \mathcal{M}^*$, with $z_1,...,z_{k+l-1} \in \mathfrak{O}(\mathcal{M}) \cup \{a_1,...,a_n\}$. Then by Lemma \ref{eqLemmaRelationCompositionA} it holds that 
$$
\eta^*(z_1) \oplus ... \oplus \eta^*(z_{k+l-1}) \in (\eta^*(R) \oplus \mathbb{R}^{l-1}) \cap (\mathbb{R}^{k-1} \oplus \eta^*(S)) 
$$
This means that there exist tuples $(a_1^1,...,a_{k+l-1}^1),\allowbreak...,\allowbreak(a_1^q,...,a_{k+l-1}^q)$, with each $a_i^j$ from $\mathfrak{O}(\mathcal{M})$,  such that for each $i$ it holds that $R_1(a_1^i,...,a_k^i) \in M$ and $R_2(a_k^i,...,a_{k+l-1}^i) \in \mathcal{M}$ and such that for $j\in \{1,...,k+l-1\}$ it holds that
\begin{align*}
\eta^*(z_j) = \lambda_1 \eta^*(a_j^1) + ... +  \lambda_q \eta^*(a_j^q)
\end{align*}
where $\lambda_1,...,\lambda_q > 0$ and $\sum_i \lambda_i=1$. Since $\mathcal{M}$ is a model of $\sigma$ it must hold that each $(a_{s_1}^i,...,a_{s_{q}}^i)$ is an instance of $R_3$, and thus for each $i\in\{1,...,q\}$ it holds that
$$
\eta^*(a_{s_1}^i)\oplus ... \oplus \eta^*(a_{s_q}^i) \in \eta^*(R_3)
$$
Since $\eta^*(R_3)$ is convex, we thus also have that
$$
(\sum_i \lambda_i \eta^*(a_{s_1}^i)\oplus ... \oplus \eta^*(a_{s_q}^i)) \in \eta^*(R_3)
$$
or equivalently
$$
\eta^*(z_{s_1})\oplus ... \oplus \eta^*(z_{s_q}) \in \eta^*(S)
$$
which is what we needed to show.
\end{proof}

The proof for datalog rules with more than two atoms in the body follows from the fact that such rules can be rewritten as datalog programs with only two atoms in each rule body, by introducing a number of fresh relation symbols. This process allows us to prove the following lemma.
\begin{lemma}\label{lemmaRulesManyBodyPredicate}
Let $\mathcal{M}^*$ be defined as above, and let $\sigma \in \Sigma$ be a quasi-chained datalog rule of the following form:
$$
R_1(A_1^1,...,A^1_{n_1}) \wedge ... \wedge R_k(A_1^k,...,A^k_{n_k}) \rightarrow S(B_1,...,B_l)
$$
where $\{B_1,...,B_l\} \subseteq \{A_1^1,...,A^k_{n_k}\}$. Then $\mathcal{M}^*$ is a model of $\sigma$.
\end{lemma}
\begin{proof}
Let us consider fresh relation names $T_{2},\allowbreak T_{3},\allowbreak...,\allowbreak T_{k-1}$ and the following set of rules:
\begin{align}
    &R_1(A_1^1,...,A^1_{n_1}) \wedge R_2(A_1^2,...,A^2_{n_2}) \label{eqLemmaNpredicatesRuleSet1}\\
    &\quad\quad\quad \rightarrow T_2(A_1^1,...,A^1_{n_1},A_1^2,...,A^2_{n_2}) \notag\\
    &T_2(A_1^1,...,A^2_{n_1}) \wedge R_3(A_1^3,...,A^3_{n_3}) \\
    &\quad\quad\quad \rightarrow T_3(A_1^1,...,A^1_{n_1},A_1^2,...,A^2_{n_2},A_1^3,...,A^3_{n_3})\notag\\
    &...\notag\\
    &T_{k-1}(A_1^1,...,A^{k-1}_{n_{k-1}}) \wedge R_k(A_1^k,...,A^k_{n_k})\label{eqLemmaNpredicatesRuleSet2}\\
    &\quad\quad\quad \rightarrow S(B_1,...,B_l) \notag
\end{align}
For the ease of presentation we will assume that $A^1_{n_1} = A_1^2$, $A^2_{n_2}= A_1^3, ..., A^{k-1}_{n_{k-1}}=A_1^k$.  However, it is easy to see that the same argument can be applied to other types of quasi-chained rules.

Let $\eta^{+}$ be the extension of $\eta$ to the relation symbols $T_2,...T_{k-1}$, defined as follows. We consider the following rules:
\begin{align*}
    \eta^+(T_2) &= (\eta(R_1) \oplus \mathbb{R}^{m(n_1-1)}) \cap (\mathbb{R}^{m(n_2-1)} \oplus \eta^*(R_2)) \\ 
    \eta^+(T_3) &= (\eta(T_2) \oplus \mathbb{R}^{m(n_1+n_2-1)})\\
    &\quad\quad\quad \cap (\mathbb{R}^{m(n_3-1)} \oplus \eta^*(R_3))\\ 
    &...\\
    \eta^+(T_{k-1}) &= (\eta(T_{k-2}) \oplus \mathbb{R}^{m(\sum_{i=1}^{k-2}n_i -1)})\\
    &\quad\quad\quad \cap (\mathbb{R}^{m(n_{k-1}-1)} \oplus \eta^*(R_{k-1})) 
\end{align*}
Let $\mathcal{M}^+ = \phi(\eta^+)$. Note that $\mathcal{M}$ and $\mathcal{M}^+$ only differ in the fact that instances of the relation names $T_2,...T_{k-1}$ have been added. Let $\eta^{+*}$ be the extension of $\eta^+$ to $\mathfrak{O}(\mathcal{M}) \cup \{a_1,...,a_n\}$ defined by $\eta^{+*}(a_i)=\mathbf{v_i}$. By Lemma \ref{lemmaRulesTwoBodyPredicate} it follows that $\mathcal{M}^{+*} = \phi(\eta^{+*})$ is a model of \eqref{eqLemmaNpredicatesRuleSet1}--\eqref{eqLemmaNpredicatesRuleSet2}. This means that $\mathcal{M}^{+*}$ is also a model of $\sigma$ since the latter rule can be deduced from  \eqref{eqLemmaNpredicatesRuleSet1}--\eqref{eqLemmaNpredicatesRuleSet2}. Since $\mathcal{M}^{+*}$ and $\mathcal{M}{^*}$ only differ in the fact that the former contains instances of $T_2,...,T_{k-1}$, it follows that $\mathcal{M}^*$ is a model of $\sigma$.
\end{proof}

In the previous lemmas, we have only considered variables as terms. However, it is easy to verify that the proofs remain valid if some (or all) of the terms are constants. Similarly, note that the definition of quasi-chainedness allows for rules where a shared variable is repeated, such as:
$$
R_1(X,Y,Y) \wedge R_2(Y,Y,Y,Z) \rightarrow S(X,Y)
$$
It is easy to verify that our results still hold for such rules. Finally, it is also straightforward to verify that the result is valid for constraints.

If $\Sigma$ contains non-datalog rules, we need to prove that it is always possible to extend $\eta^*$ to any nulls that are required to define a model of $(\Sigma, D \cup \phi(\eta^*))$. Note the recursive nature of this condition, which is intuitively due to the fact that adding nulls may require us to introduce additional nulls. 

\begin{lemma}\label{lemmaExistentialRule}
Let $\eta^*$ be defined as before. There is a (potentially infinite) set of points ${\{\mathbf{w_1},\mathbf{w_2},...\}} \subseteq \mathbb{R}^m$ and fresh nulls $x_1,x_2,...$ such that the extension $\eta^{**}$ of $\eta^*$ to $x_1,x_2,...$ defined by $\eta^{**}(x_i) = {\bf w_i}$, is such that $\mathcal M^{**} = \phi(\eta{^{**}})$ is a model of $(\Sigma, D \cup \phi(\eta^{**}))$.
\end{lemma}
\begin{proof}
We construct $\eta^{**}$, and the associated model $\mathcal{M}^{**}=\phi(\eta^{**})$ in an incremental fashion, where initially we set $\eta^{**}=\eta^*$. Note that we already have that $\mathcal{M}^{**} \models D$, since $\eta^{**}$ is an extension of $\eta$, and that we trivially have $\mathcal{M}^{**}\models \phi(\eta^{**})$. What remains is to extend $\eta^{**}$ such that $\mathcal{M}^{**} \models \Sigma$.
Suppose that a rule $\sigma$ of the following form were not satisfied by $\mathcal{M}^{**}$:
\begin{align}\label{eqExistentialRuleOneBodyPredicate}
R(A_1,...,A_k) \rightarrow \exists B_1,...,B_l. S(A_{s_1},...,A_{s_t},B_1,...,B_l)
\end{align}
where $s_1,...,s_t$ are (not necessarily distinct) indices from $\{1,...,k\}$. Note that for the ease of presentation, we consider a rule with a single atom in the body. However, by using an entirely analogous argument as in Lemmas \ref{lemmaRulesTwoBodyPredicate} and \ref{lemmaRulesManyBodyPredicate} we can extend this result for rules with more than one atom in the body.

The fact that $\sigma$ is not satisfied by $\mathcal{M}^{**}$ means that there are $z_1,...,z_k \in \mathfrak{O}(\mathcal{M}^{**})$ such that $R(z_1,...,z_k) \in {\mathcal{M}}^{**}$ and
\begin{align*}
\mathcal{M}^{**}\not \models \exists B_1,...,B_l. S(z_{s_1},...,z_{s_t},B_1,...,B_l)
\end{align*}
From $R(z_1,...,z_k) \in {\mathcal{M}}^{**}$ we have that:
$$
\eta^{**}(z_1)\oplus ... \oplus \eta^{**}(z_k) \in \eta^{**}(R)
$$
Since $\eta^{**}(R)=\eta(R)$ this means that there exist instances $(a_1^1,...,a_k^1),\allowbreak...,\allowbreak(a_1^q,...,a_k^q)$ of $R$, with each $a_i^j$ from $\mathfrak{O}(\mathcal{M})$, such that for $j\in \{1,...,k\}$ it holds that
\begin{align*}
\eta^*(z_j) = \lambda_1 \eta^*(a_j^1) + ... +  \lambda_q \eta^*(a_j^q)
\end{align*}
where $\lambda_1,...,\lambda_q > 0$ and $\sum_i \lambda_i=1$. Since $\mathcal{M}$ is a model of $\sigma$ it must hold that there are objects $u_1^1,\allowbreak ...,\allowbreak u_l^1,\allowbreak ...,\allowbreak u_1^q,\allowbreak ...,\allowbreak u_l^q$ in $\mathfrak{O}(\mathcal{M})$ such that, for each $i\in \{1,...,q\}$, $(a_{s_1}^i,...,a_{s_t}^i,u_1^i,...,u_l^i)$ is an instance of $S$. We then have that $\eta^{**}(S)$ contains the following point:
$$
\sum_i \lambda_i \left(\eta^{**}(a_{s_1}^i){\oplus} ... {\oplus}\, \eta^{**}(a_{s_t}^i) {\oplus}\,\eta^{**}(u_1^i){\oplus} ... {\oplus}\, \eta^{**}(u_l^i)\right)
$$
Now we introduce fresh nulls $c_1,...,c_l$ and define $\eta^{**}(c_j) = \sum_i \lambda_i \eta^{**}(u_j^i)$ (unless $q=1$ in which case we can replace $c_j$ by $u_j^1$). Then we have that $\eta^{**}(S)$ contains:
$$
\eta^{**}(z_{s_1})\oplus ... \oplus \eta^{**}(z_{s_t}) \oplus \eta^{**}(c_1) \oplus ... \oplus \eta^{**}(c_l)
$$
which means that $M^{**}$ contains $S(z_1,...,z_k,c_1,...,c_l)$ and thus satisfies the head of the rule \eqref{eqExistentialRuleOneBodyPredicate}. We repeat this process for each instance $(z_1,...,z_k)$ of $R$, and for each rule $\sigma$ in $\Sigma$. In general, this process may not terminate, although the set of added nulls is clearly countable. It is furthermore clear that the infinite interpretation $\mathcal{M}^{**}$ which is constructed in this way is indeed a model of $\sigma$.
\end{proof}

\fi

\end{document}